\documentclass[12pt]{article}
\usepackage[utf8]{inputenc}

\usepackage{pdfsync}
\usepackage{fullpage}
\newif\ifdraft \drafttrue
\newif\iffull \fulltrue

\usepackage{graphicx}
\usepackage{xcolor}
\usepackage[english]{babel}
\usepackage{csquotes}

\usepackage{xspace}
\usepackage[T1]{fontenc}

\usepackage{tabularx}
\usepackage{booktabs}
\usepackage{multirow}
\usepackage{multicol}
\usepackage{stfloats}
\usepackage{caption} 
\usepackage{subcaption}
\usepackage{algorithmic}
\usepackage[ruled]{algorithm2e}
\usepackage{amsmath}
\usepackage{amsthm}
\usepackage{amsfonts}
\usepackage{amssymb}
\usepackage{mathtools}
\usepackage{bbm}
\usepackage{xspace}

\usepackage{microtype}
\usepackage{graphicx}

\usepackage{color-edits}
\addauthor{blue}{sw}

\theoremstyle{definition}
\newtheorem{definition}{Definition}[section]
\newtheorem{theorem}[definition]{Theorem}
\newtheorem{lemma}[definition]{Lemma}

\newtheorem{proposition}[definition]{Proposition}
\newtheorem*{proposition*}{Proposition}

\newcommand{\ours}{\mbox{{\sf PMW\textsuperscript{Pub}}}\xspace}
\newcommand{\mwem}{\mbox{{\sf MWEM}}\xspace}
\newcommand{\dq}{\mbox{{\sf DualQuery}}\xspace}
\newcommand{\hdmm}{\mbox{{\sf HDMM}}\xspace}
\newcommand{\fem}{\mbox{\sf FEM}\xspace}

\DeclareMathOperator\supp{supp}
\DeclareMathOperator\diver{\text{D}}

\DeclareMathOperator*{\argmin}{arg\,min}

\newcommand{\eps}{\varepsilon}
\newcommand{\cX}{\mathcal{X}}
\newcommand{\cQ}{\mathcal{Q}}

\newcommand{\cR}{\mathcal{R}}
\newcommand{\cM}{\mathcal{M}}

\newcommand{\privD}{\widetilde {D}}
\newcommand{\pubD}{\widehat{D}}
\newcommand{\pubDomain}{\widehat{\cX}}

\newcommand{\pr}[1]{\text{Pr}\left[ #1 \right]}

\renewcommand{\cite}[1]{\citep{#1}}

 \usepackage{xcolor}

\newcommand{\tl}[1]{ {\color{orange}TL: #1}}
\newcommand{\ts}[1]{ {\color{blue}TS: #1}}

\usepackage[style=alphabetic,backend=bibtex,maxalphanames=10,maxbibnames=20,maxcitenames=10,giveninits=true,doi=false,url=true]{biblatex}
\newcommand*{\citet}[1]{\AtNextCite{\AtEachCitekey{\defcounter{maxnames}{2}}}
\textcite{#1}}

\newcommand*{\citep}[1]{\cite{#1}}
\newcommand{\citeyearpar}[1]{\cite{#1}}

\usepackage{hyperref}
\usepackage{cleveref}

\allowdisplaybreaks

\begin{document}

\title{Leveraging Public Data for\\Practical Private Query Release}

\newcommand{\asdfasdf}{~~~~~~~~}% adding some space to get 2 authors on first row and 3 on second, otherwise it's 4 on first and 1 on second row. 
\author{
\asdfasdf
Terrance Liu\thanks{Carnegie Mellon University ~\dotfill~ \texttt{terrancl@andrew.cmu.edu}}
\asdfasdf\\
\and 
\asdfasdf
Giuseppe Vietri\thanks{University of Minnesota ~\dotfill~ \texttt{vietr002@umn.edu}}
\asdfasdf\\
\and
Thomas Steinke\thanks{Google Research, Brain Team ~\dotfill~ \texttt{dpmw-pub@thomas-steinke.net}}
\and
Jonathan Ullman\thanks{Northeastern University ~\dotfill~ \texttt{jullman@ccs.neu.edu}}
\and
Zhiwei Steven Wu\thanks{Carnegie Mellon University ~\dotfill~ \texttt{zstevenwu@cmu.edu}}
}
%\icmlaffiliation{cmu}{Carnegie Mellon University, Pittsburgh, PA, USA}
%\icmlaffiliation{umn}{University of Minnesota, Minneapolis, MN, USA}
%\icmlaffiliation{goo}{Google, Mountain View, CA, USA}
%\icmlaffiliation{neu}{Northeastern University, Boston, MA, USA}

\date{}

\maketitle

%%%%%%%%%%%%%%%%%%%%%%%%%%%%%%%%%%%%%%%%%%
%%%%%%%%%%%%%%%%%%%%%%%%%%%%%%%%%%%%%%%%%%
%%%%%%%%%%%%%%%%%%%%%%%%%%%%%%%%%%%%%%%%%%
%%%%%%%%%%%%%%%%%%%%%%%%%%%%%%%%%%%%%%%%%%

\begin{abstract}

In many statistical problems, incorporating priors can significantly improve performance. However, the use of prior knowledge in differentially private query release has remained underexplored, despite such priors commonly being available in the form of public datasets, such as previous US Census releases. With the goal of releasing statistics about a private dataset, we present \ours, which---unlike existing baselines---leverages public data drawn from a related distribution as prior information. We provide a theoretical analysis and an empirical evaluation on the American Community Survey (ACS) and ADULT datasets, which shows that our method outperforms state-of-the-art methods. Furthermore, \ours\ scales well to high-dimensional data domains, where running many existing methods would be computationally infeasible.

\end{abstract}

\section{Introduction}

As the collection and distribution of private information becomes more prevalent, controlling privacy risks is becoming a priority for organizations that depend on this data. Differential privacy \cite{DworkMNS06} is a rigorous criterion that provides meaningful guarantees of individual privacy while allowing for trade-offs between privacy and accuracy.  It is now deployed by organizations such as Google, Apple, and the US Census Bureau. In this work, we study the problem of \emph{differentially private query release}, specifically generating a \emph{private synthetic dataset}: a new dataset in which records are ``fake'' but the statistical properties of the original data are preserved.  In particular, the release of summary data from the 2020 US Decennial Census---one of the most notable applications of differential privacy \cite{Abowd18}---can be framed as a private query release problem.  

In practice, generating accurate differentially private synthetic datasets is challenging without an excessively large private dataset, and a promising avenue for improving these algorithms is to find methods for incorporating \emph{prior information} that lessen the burden on the private data.  In this paper we explore using public data as one promising source of prior information that can be used without regard for its privacy.\footnote{The public data may have been derived from private data, but we refer to it as ``public'' for our purposes as long as the privacy concerns have already been addressed.} For example, one can derive auxiliary data for the 2020 US Census release from already-public releases like the 2010 US Census. Similarly, the Census Bureau's American Community Survey has years of annual releases that can be treated as public data for future releases. Alternatively, once national-level statistics are computed and released, they can serve as public data for computing private statistics over geographic subdivisions, such as states and counties. Indeed, such a hierarchy of releases is part of the {\sf TopDown} algorithm being developed for the 2020 US Census \cite{AbowdASKLMS19}.

Existing algorithms for private query release do not incorporate public data.  While there is theoretical work on \emph{public-data-assisted private query release} \cite{bassily2020private}, it crucially assumes that the public and private data come from the same distribution and does not give efficient algorithms.

\paragraph{Our Contributions}
Therefore, in light of these observations, we make the following contributions:
\begin{enumerate}
    \item We initiate the study of using public data to improve private query release in the more realistic setting where the public data is from a distribution that is \textit{related but not identical} to the distribution of the private data.
    
    \item We present (Private) Multiplicative Weights with Public Data (\ours), an extension of \mwem \cite{HardtLM12} that incorporates public data.
    
    \item We show that as a side benefit of leveraging public data, \ours is computationally efficient and therefore is practical for much larger problem sizes than \mwem.
    
    \item We analyze the theoretical privacy and accuracy guarantees of \ours.
    
    \item We empirically evaluate \ours on the American Community Survey (ACS) data to demonstrate that we can achieve strong performance when incorporating public data, even when public samples come from a different distribution. 
    % For example, relative to baseline methods, \ours gives improvements between $1.6 \times$ and $4.6 \times$ on the ACS PA-18 data (depending on the given privacy budget).
    
\end{enumerate}

%Typically, differentially private query release is studied in a setting where the algorithm is given only the private data and has no other sources of information. However, in many realistic settings, there exist auxiliary sources of data for which we need not be concerned about privacy---we refer to this as \textit{public} data.

%Given that an algorithm can freely access such public data without risk to its private data, improving statistical accuracy by incorporating prior knowledge from public data is a promising -- but underexplored -- area of research. %Motivated by these observations, we focus specifically on \textit{Public-data-Assisted Private (PAP)} query release \cite{bassily2020private}, a relaxed setting of the query release problem in which query-release algorithms have access to both public and private samples.

%This auxiliary public data can be exploited to improve statistical accuracy.
%For this to be useful, the public data need not be statistically identical to the private data. As long as the public data is vaguely similar to the private data, it can potentially serve as a useful starting point for query release.

% In addition, we provide examples of failure cases and practical strategies for avoiding them.

\subsection{Related Work} 
Our work relates to a growing line of research that utilizes publicly available data for private data analyses in the setting where the public and private data come from the same distribution. For private query release, \citet{bassily2020private} prove upper and lower bounds on the number of public and private samples needed, and \citet{alon2019limits} do the same for binary classification.  Neither of these works, however, give computationally efficient algorithms.  Other works consider a more general problem, \emph{private prediction}, where the public data is unlabeled and the private data is used to label the public data~\cite{bassily2018model, dwork2018privacy, nandi2020privately,PapernotSMRTE18}. \citet{beimel2013private} consider the \emph{semi-private} setting where only a portion of examples require privacy for the labels.

These prior works are limited by the strong assumption that the public and private data are drawn from the same distribution. One notable exception is the recent work of \citet{bassily2020learning} on supervised learning, in which the authors assume that the public and private data can be labeled differently but have the same marginal distribution without labels. Given that their problem is trivial otherwise, they focus solely on the setting where the public dataset is smaller than the private dataset. However, if the public data comes from a different distribution (as is the case in our experiments), the setting in which the size of the public dataset is similar to or larger than that of the private dataset becomes interesting.

Finally, \citet{ji2013differential} propose a method that, like \ours, reweights a support derived from a public dataset (via importance weighting). However, while their method does not rely on the assumption that the public and private data come from the same distribution, \citet{ji2013differential} do not make this distinction in their theoretical analysis or discussion. Moreover, unlike the algorithm presented in this work, their method is not tailored to the problem of query release.\footnote{
See Appendix \ref{app:other_baselines} for an additional discussion.
}

\section{Preliminaries}\label{sec:prelims}

We consider a data domain $\cX=\{0,1\}^d$ of dimension $d$, a private dataset $\privD \in \cX^n$ consisting of the data belonging to $n$ individuals, and a class of statistical linear queries $\cQ$. Our final objective is to generate a synthetic dataset in a privacy-preserving way that matches the private data's answers. Consider a randomized mechanism $\cM:\cX^n\rightarrow \cR$ that takes as input a private dataset $\privD$ and computes a synthetic dataset $X\in\cR$, where $\cR$ represents the space of possible datasets. Given a set of queries $\cQ$, we say that the max error of a synthetic dataset $X$ is given by $\max_{q\in \cQ } | q(\privD) - q(X)|$. 
% The is to approximately answer a large class of statistical queries $\cQ$ about $\privD$ while preserving the privacy of the $n$ individuals. 
% We say that $a\in [0,1]$ is an $\alpha$-approximate answer to some query $q_{\phi} \in \cQ$ if it satisfies $|a-q_\phi(\privD)| \le \alpha$ for some accuracy parameter $\alpha>0$ and 

We begin with the definition of a statistical linear query:
%
%%%%%%%%%%%%%%%%%%
%% Def 
%%%%%%%%%%%%%%%%%
\begin{definition}[Statistical linear query] 
Given a predicate $\phi$ and a dataset $D$, the linear query $q_{\phi}: \mathcal{X}^{n} \rightarrow[0,1]$ is defined by
\begin{align*}
    q_{\phi}(D)=\frac{1}{|D|} \sum_{x \in D} \phi(x)
\end{align*}
Defining a dataset instead as a distribution $A$ over the domain $\cX$, the definition for a linear query $q_{\phi}$ then becomes $q_\phi(A) = \sum_{x\in\cX}\phi(x)A(x)$.
\end{definition}

One example of a statistical query class is $k$-way marginal queries, which we define below.
%
%%%%%%%%%%%%%%%%%
%% Def 
%%%%%%%%%%%%%%%%%
\begin{definition}[$k$-way marginal query] Let the data universe with $d$ categorical attributes be $\mathcal{X}=\left(\mathcal{X}_{1} \times \ldots \times \mathcal{X}_{d}\right)$, where each $\mathcal{X}_{i}$ is the discrete domain of the $i$th attribute. A $k$-way marginal query $ \phi_{S,y}$ is a linear query specified by a set of $k$ attributes
% $A = \{(A_i)_{i\in[k]} \mid A_{1} \neq \ldots \neq A_{k} \in[d]\}$
$S\subseteq [d]$ (with $|S| = k$)
and target 
% $y \in\left(\mathcal{X}_{1} \times \ldots \times \mathcal{X}_{k}\right)$,
$y \in \prod_{i\in S} \cX_i$
such that for all $x\in \cX$
\begin{align*}
    \phi_{S,y}(x)= 
    \left\{\begin{array}{ll}
    1 & :   \forall j \in S \quad x_j = y_j \\ 
    % x_{a_1}=y_{1} \wedge \ldots \wedge x_{a_k}=y_{k} \\
    0 & : \textrm{otherwise}
\end{array}\right.
\end{align*}
where $x_{i} \in \mathcal{X}_{i}$ means the $i$th attribute of record $x \in \mathcal{X}$.
We define a \textit{workload} as the set of marginal queries given by a set of attributes $S$. The workload given by attributes in $S$ has a total of $\prod_{i\in S} \left|\mathcal{X}_{i}\right|$ marginal queries.
\end{definition}

Although we evaluate on $k$-way marginal queries in our experiments, we provide theoretical results that hold for any class of linear queries. 

%
%%%%%%%%%%%%%%%%%
%% Def 
%%%%%%%%%%%%%%%%%
\begin{definition}[$\ell_1$-sensitivity]\label{def:sensitivity}
 The $\ell_1$-sensitivity of a function $f:\cX^*\rightarrow \mathbb{R}^k$ is
 \begin{align*}
     \Delta f = \max_{\text{neighboring } D,D'} \| f(D) - f(D')\|_1
 \end{align*}
\end{definition}

In the context of statistical queries, the $\ell_1$-sensitivity of query captures the effect of changing an individual in the dataset and is useful for determining the amount of perturbation required for preserving privacy. 

In our setting, we have access to a public dataset $\pubD\in \cX^m$ containing the data of $m$ individuals that we can use without privacy constraints. This dataset defines a public data domain, denoted by $\pubDomain\subset \cX$, which consists of all unique rows in $\pubD$.  We assume that both the public and private datasets are i.i.d.~samples from different distributions and use the R\'{e}nyi divergence, which we define below, as a measure for how close the two distributions are.
\begin{definition}[R\'{e}nyi divergence]\label{def:renyidiver}
 Let $\mu$ and $\nu$ be probability
distributions on $\Omega$. For $\alpha\in(1,\infty)$, we define the R\'{e}nyi divergence of order $\alpha$ between $\mu$ and  $\nu$ as
\begin{align*}
   \diver_{\alpha} (\mu\parallel\nu) = \frac{1}{1-\alpha} \log \sum_{x\in\Omega} \mu(x)^\alpha\nu(x)^{1-\alpha} 
\end{align*}
\end{definition}

The R\'{e}nyi divergence is also used in the definition of privacy that we adopt. The output of a randomized mechanism $\cM:\cX^*\rightarrow \cR$ is a privacy preserving-computation if it satisfies concentrated differential privacy (CDP) \cite{DworkR16,BunS16}:
\begin{definition}[Concentrated DP]\label{def:zCDP}
 A randomized mechanism $M:\cX^n \rightarrow \cR$ is $\frac12\tilde{\eps}^2$-CDP, if for all neighboring datasets $D,D'$ (i.e., differing on a single person), and for all $\alpha\in(1,\infty)$,
 \begin{align*}
     \diver_\alpha(\cM(D)\parallel\cM(D')) \leq \frac12\tilde{\eps}^2\alpha
 \end{align*}
 where $\diver_\alpha(\cM(D)\parallel\cM(D'))$ is the R\'{e}nyi divergence between the distributions of $\cM(D)$ and $\cM(D')$.
\end{definition}
Two datasets are \emph{neighboring} if you can obtain one from the other by changing the data of one individual.
Definition \ref{def:zCDP} says that a randomized mechanism computing on a dataset satisfies zCDP if its output distribution does not change by much in terms of R\'{e}nyi divergence when a single user in the dataset is changed. 
%% GV: I combined these two paragraphs. 
Finally, any algorithm that satisfies zCDP also satisfies (approximate) differential privacy \citep{DworkMNS06,DworkKMMN06}:
\begin{definition}[Differential Privacy (DP)] A randomized algorithm $\cM: \cX^* \to \cR$ satisfies $(\eps, \delta)$-differential privacy (DP) if for all neighboring databases $D, D'$, and every event $E \subseteq \cR$, we have
	\[\Pr[\cM(D) \in E] \le e^{\eps} \Pr[\cM(D') \in E] + \delta.\]
If $\delta = 0$, we say that $\cM$ satisfies pure (or pointwise) $\eps$-differential privacy.
\end{definition}

% \begin{definition}[Differential Privacy (DP)]
% A randomized algorithm $\mathcal{M}: \mathcal{X}^{*} \rightarrow \mathcal{R}$ satisfies $(\varepsilon, \delta)$-differential privacy (DP) if for all databases $x, x^{\prime}$ differing at most one entry and every measurable subset $ \mathcal{S} \subseteq \mathcal{R}$, we have that
% %
% \begin{align*}
%     \textrm{Pr}[\mathcal{M}(x) \in S] \leq e^{\varepsilon} \textrm{Pr}\left[\mathcal{M}\left(x^{\prime}\right) \in S\right]+\delta
% \end{align*}
% If $\delta=0$, we say that $\mathcal{M}$ satisfies $\varepsilon$-differential privacy.
% \end{definition}

\section{Public Data Assisted \mwem}\label{sec:mwem_pub}

In this section, we revisit \mwem and then introduce \ours, which adapts \mwem to leverage public data.

\subsection{\mwem}

\mwem \cite{HardtLM12} is an approach to answering linear queries that combines the multiplicative weights update rule for no-regret learning and the exponential mechanism \cite{mcsherry2007mechanism} for selecting queries. It is a simplification of the private multiplicative weights algorithm \cite{HR10}. \mwem maintains a distribution over the data domain $\cX$ and iteratively improves its approximation of the distribution given by the private dataset $\tilde D$. At each iteration, the algorithm privately selects a query $q_t$ with approximately maximal error using the exponential mechanism and approximates the true answer to the query with Laplace noise \cite{DworkMNS06}. \mwem then improves the approximating distribution using the multiplicative weights update rule. This algorithm can be viewed as a two-player game in which a data player updates its distribution $A_t$ using a no-regret online learning algorithm and a query player best responds using the exponential mechanism. 

Our choice of extending \mwem stems from the following observations: (1) in the usual setting without public data, \mwem attains worst-case theoretical guarantees that are nearly information-theoretically optimal \cite{BunUV18}; (2) \mwem achieves state-of-the-art results in practice when it is computationally feasible to run; and (3) \mwem can be readily adapted to incorporate ``prior'' knowledge that is informed by public data. 

However, maintaining a distribution $A$ over a data domain $\cX=\{0,1\}^d$ is intractable when $d$ is large, requiring a run-time of $O(n|\cQ|+T|\cX||\cQ|))$, which is exponential in $d$ \cite{HardtLM12}. Moreover, \citet{UllmanV11} show that computational hardness is inherent for worst-case datasets, even in the case of 2-way marginal queries. Thus, applying \mwem\ is often impractical in real-world instances, prompting the development of new algorithms \cite{gaboardi2014dual, vietri2020new} that bypass computational barriers at the expense of some accuracy. 

\subsection{\ours}

\begin{algorithm}[tb]
\caption{\ours}
\label{alg:framework}
\begin{algorithmic}
  \STATE {\bfseries Input:} Private dataset $\privD \in \cX^n$, public dataset $\pubD \in \cX^m$, query class $\cQ$, privacy parameter $\tilde{\varepsilon}$, number of iterations $T$.
  \STATE Let the domain be $\pubDomain = \mathrm{supp}(\pubD)$.
  \STATE Let size of the private dataset be $n = |\privD|$. 
  \STATE Let $A_0$ be the distribution over $\pubDomain$ given by $\pubD$ 
  \STATE Initialize $\varepsilon_0 = \frac{\tilde{\varepsilon}}{\sqrt{2T}}$.
    \FOR{$t = 1 $ {\bfseries to} $T$}
        \STATE \textbf{Sample} query $q_t\in\cQ$ using the \emph{permute-and-flip mechanism} or \emph{exponential mechanism} -- i.e.,
        \[ \Pr[q_t] \propto \exp\left(\frac{\varepsilon_0 n}{2}|q(A_{t-1}) - q(\privD)|\right)\]
        \STATE \textbf{Measure:} Let $a_t = q_t(\privD) + \mathcal{N}\left(0,1/n^2\varepsilon_0^2\right)$. (But, if $a_t<0$, set $a_t=0$; if $a_t>1$, set $a_t=1$.)
        \STATE \textbf{Update:} Let $A_t$ be a distribution over $\pubDomain$ s.t.
        \[A_{t}(x) \propto A_{t-1}(x)\exp{\left( q_t(x)\left(a_t - q_t(A_{t-1})\right) / 2 \right)}.\]
    \ENDFOR
    \STATE \textbf{Output:} $A = \text{avg}_{t \in [T]} A_{t-1}$
\end{algorithmic}
\end{algorithm}

We now introduce \ours in Algorithm \ref{alg:framework}, which adapts \mwem to utilize public data in the following ways:

First, the approximating distribution $A_t$ is maintained over the public data domain $\pubDomain$ rather than $\cX$, implying that the run-time of \ours is $O(n|\cQ|+T|\pubDomain||\cQ|))$. Because $|\pubDomain|\le m$ is often significantly smaller than $|\cX|$, \ours offers substantial improvements in both run-time and memory usage, scaling well to high-dimensional problems.
    
Second, $A_0$ is initialized to the distribution over $\pubDomain$ given by $\pubD$. By default, \mwem initializes $A_0$ to be uniform over the data domain $\cX$. This na\"ive prior is appropriate for worst-case analysis, but, in real-world settings, we can often form a reasonable prior that is closer to the desired distribution. Therefore, \ours initializes $A_0$ to match the distribution of $\pubD$ under the assumption that the public dataset's distribution provides a better approximation of $\privD$.

In addition, we make two additional improvements:

\textbf{Permute-and-flip Mechanism.} We replace the \textit{exponential mechanism} with the \textit{permute-and-flip mechanism} \cite{McKennaS20}, which like the \textit{exponential mechanism} runs in linear time but whose expected error is never higher.

\textbf{Gaussian Mechanism.} When taking measurements of sampled queries, we add Gaussian noise instead of Laplace noise. The Gaussian distribution has lighter tails, and in settings with a high degree of composition, the scale of Gaussian noise required to achieve some fixed privacy guarantee is lower \cite{CanonneKS20}. Privacy guarantees for the \textit{Gaussian mechanism} can be cleanly expressed in terms of concentrated differential privacy and the composition theorem given by \citet{BunS16}.

% \REPEAT

% \FOR{$t = 1 $ {\bfseries to} $T$}{
%     \noindent \textbf{Sample } a query $q_t\in\cQ$ using the \emph{exponential mechanism} with epsilon value $\varepsilon_0/2$ and the score function:
%     \[
%     S_t(\privD, q) = |q(A_{t-1}) - q(\privD)|
%     \]
%     \noindent \textbf{Measure:} Let $a_i = q_t(\privD) + \textrm{Lap}(\varepsilon_0/2)$  \\
%     \noindent \textbf{Update:} Let $A_t$ be a distribution over $\pubDomain$ s.t
%     \[
%     A_{t}(x) \propto A_{t-1}(x)\exp{\left( q_t(x)\left(a_i - q_t(A_{t-1})\right) \right)}
%     \]
% }
% \ENDFOR
    
    % 
% % \INPUT A dataset $D\in \cX^n$, query class $\cQ$, number of rounds $T$, target privacy $\rho$.
% % % \KwIn{A dataset $D$, Queryset $\cQ$, No-regret algorithm $\mathcal{A}$, Number of rounds $T$}
% % % \STATE Initialize $\eps_0 $ such that $\eps = T\eps_0^2/2 + \eps_0\sqrt{2T \log{(1/\delta)}}$
% % \STATE Initialize $\rho_0 = \rho/T$. Get initial sample $q_0\in\cQ$ uniformly at random.
% % % \STATE Let $\cM_E$ be the exponential mechanism.

% \input{docs/theory}
\begingroup
\newcommand{\pp}[1]{\left( #1 \right)}
\newcommand{\bb}[1]{\left[ #1 \right]}
\newcommand{\cc}[1]{\left\{ #1 \right\}}

\newcommand{\Dhat}{\widehat{D}}
\newcommand{\Dtil}{\widetilde{D}}
\newcommand{\error}{\text{error}}
\newcommand{\Xhat}{\widehat{\cX}}

\newcommand{\Dsupp}{\Xhat}
\newcommand{\Ddist}{\mu_{\Dhat}}

\newcommand{\besterror}{\alpha^{-}_{\Dhat}}
\newcommand{\worsterror}{\alpha^{+}_{\Dhat}}

\newcommand{\errorLBfunc}{f}
\newcommand{\DeltaHat}{\widehat{\Delta}}
\newcommand{\DeltaTil}{\widetilde{\Delta}}

\newcommand{\muhat}{\widehat{\mu}}
\newcommand{\mutil}{\widetilde{\mu}}
\newcommand{\KL}{d_{\text{KL}}}

\DeclarePairedDelimiter\abs{\lvert}{\rvert}%
\DeclarePairedDelimiter\norm{\lVert}{\rVert}%
% \newcommand{\l}{\left(}
% \newcommand{\r}{\right)}
% Swap the definition of \abs* and \norm*, so that \abs
% and \norm resizes the size of the brackets, and the 
% starred version does not.
\makeatletter
\let\oldabs\abs
\def\abs{\@ifstar{\oldabs}{\oldabs*}}
\let\oldnorm\norm
\def\norm{\@ifstar{\oldnorm}{\oldnorm*}}
\makeatother

\newcommand{\ex}[2]{{\ifx&#1& \mathbb{E} \else \underset{#1}{\mathbb{E}} \fi \left[#2\right]}}

\section{Theoretical Analysis}\label{sec:theory}

In this section, we analyze the accuracy of \ours under the assumption that the public and private dataset are i.i.d. samples from two different distributions.
The support of the a dataset $X\in\cX^*$ is the set $\supp(X)=\{x \in \cX : x \in X \}$, and we denote the support of the public dataset $\Dhat $ by $\Xhat = \supp(\Dhat)$. Recall that \ours\ (Algorithm \ref{alg:framework}) takes as input a public dataset and then updates its distribution over the public dataset's support using the same procedure found in \mwem.  We show that the accuracy of \ours\ will depend on the best mixture error over the public dataset support $\Dsupp$, which we characterize using the best mixture error function $\errorLBfunc_{\Dtil,\cQ}:2^\cX\rightarrow [0,1]$ that measures
a given support's ability to approximate the private dataset  $\Dtil$ over the set of queries $\cQ$. 
% The function $\errorLBfunc_{\Dtil,\cQ}$ maps a support $S$ to the best mixture error over the support $S$ to approximate $\Dtil$ over queries $\cQ$. 
The precise definition is as follows: 
\begin{definition}
For any support  $S\in 2^\cX$, the best mixture error of $S$ to approximate a dataset $D$ over the queries $Q$ is given by the function:
\begin{align*}
    % \besterror = 
    \errorLBfunc_{D,Q}(S)=
    \min_{\mu \in \Delta(S)} 
    \max_{q\in Q}
    \left|
    q\pp{D} - 
    \sum_{x \in S} \mu_x  q(x)
    \right|
\end{align*}
where $\mu \in \Delta(S)$ is a distribution over the set $S$ with $\mu_x \ge 0$ for all $x\in S$ and $\sum_{x\in S} \mu_x = 1$.
\end{definition}
Intuitively, \ours\ reweights the public dataset in a differentially private manner to approximately match the private dataset's answers; the function $f_{\Dtil, \cQ}(\Xhat)$ captures how well the best possible reweighting on $\Xhat$ would do in the absence of any privacy constraints. While running \ours does not explicitly require calculating the best mixture error, in practice it may prove useful to release it in a privacy-preserving way. We present the following lemma, which shows that $f_{\Dtil, \cQ}(\Xhat)$ has bounded sensitivity.

\begin{lemma}
For any  support $S\in 2^\cX$ and set $Q$, 
the  best mixture error function $f_{D,Q}$ is $\tfrac{1}{n}$ sensitive. 
That is for any pair of neighboring datasets $D,D'$ of size $n$, $\abs{f_{D,Q}(S)  - f_{D',Q}(S)} \leq \tfrac{1}{n}$.
\end{lemma}

It follows that we can release $f_{\Dtil, \cQ}(\Xhat)$, using the Laplace or Gaussian mechanism with magnitude scaled by $\tfrac{1}{n}$.

We show that, if the public and private datasets are drawn from similar distributions, then, with high probability, $f_{\Dtil, \cQ}(\Xhat)$ is small. Note that the required size of the public dataset increases with the divergence between the private and public distributions.

\begin{proposition}\label{prob:mix}
Let $\mu,\nu\in\Delta(\mathcal{X})$ be distributions with $\mathrm{D}_\infty(\mu\|\nu)<\infty$. Let $\Dtil \sim \mu^n$ and $\Dhat \sim \nu^m$ be $n$ and $m$ independent samples from $\mu$ and $\nu$ respectively. Let $\Xhat$ be the support of $\Dhat$. Let $Q$ be a finite set of statistical queries $q : \mathcal{X} \to [0,1]$. Let $\alpha,\beta>0$. If $n \ge \frac{8}{\alpha^2} \log \left(\frac{4|Q|}{\beta}\right)$ and $m \ge \left(\frac{32}{\alpha^2} e^{\mathrm{D}_2(\mu\|\nu)} + \frac{8}{3\alpha} e^{\mathrm{D}_\infty(\mu\|\nu)}\right) \log \left( \frac{4|Q|+4}{\beta} \right)$, then
\[\Pr\left[ f_{\Dtil,Q}(\Xhat) \le \alpha \right] \ge 1-\beta. \]
\end{proposition}
\begin{proof}
Note that we may assume $\alpha<1$ as the result is trivial otherwise.
Let $g(x) = \mu(x)/\nu(x)$. Then $0 \le g(x) \le e^{\mathrm{D}_\infty(\mu\|\nu)}$ for all $x$ and, for $X \sim \nu$, we have $\mathbb{E}[g(X)]=1$ and $\mathbb{E}[g(X)^2] = e^{\mathrm{D}_2(\mu\|\nu)}$. Define $\omega \in \Delta(\Xhat)$ by $\omega_x = \frac{g(x)}{\sum_{x \in \Dhat} g(x)}$ for $x\in\Xhat$. Clearly $f_{\Dtil,Q}(\Xhat) \le \max_{q \in Q} \left| q(\Dtil) - \sum_{x \in \Dhat} \omega_x q(x) \right|$. 
% \gv{should the summation be over $\Xhat$?}\ts{No. We do want to incorporate the weighting of the public dataset. Perhaps there is better notation for this though?}

Fix some $q \in Q$.
By Hoeffding's inequality, \[\Pr[|q(\Dtil)-q(\mu)|\ge \alpha/4] \le 2\cdot e^{-\alpha^2n/8}.\]
For $X \sim \nu$, $\mathbb{E}[g(X)q(X)] = q(\mu)$ and $\mathsf{Var}[g(X)q(X)] \le \mathbb{E}[(g(X)q(X))^2] \le \mathbb{E}[g(X)^2] = e^{\mathrm{D}_2(\mu\|\nu)}$.
By Bernstein's inequality,
\begin{align*}
    &\Pr\left[\left|m \cdot q(\mu) - \sum_{x \in \Dhat} g(x) q(x) \right| \ge \frac{\alpha}{4} m \right] 
    \le 2 \cdot \exp\left(\frac{-\alpha^2 m}{32 \cdot e^{\mathrm{D}_2(\mu\|\nu)} + \frac{8}{3} \alpha \cdot e^{\mathrm{D}_\infty(\mu\|\nu)}}\right).
\end{align*}
Let $\hat m = \sum_{x \in \Dhat} g(x)$. Similarly,
\begin{align*}
    \Pr\left[|\hat m - m| \ge \frac{\alpha}{4}m\right] &= \Pr\left[\left|m - \sum_{x \in \Dhat} g(x) \right| \ge \frac{\alpha}{4}m \right] \\
    &\le 2 \cdot \exp\left(\frac{-\alpha^2 m}{32 \cdot \left(e^{\mathrm{D}_2(\mu\|\nu)}-1\right) + \frac{8}{3} \cdot \alpha \cdot e^{\mathrm{D}_\infty(\mu\|\nu)}}\right).
\end{align*}
If all three of the events above do not happen, then
\begin{align*}
    \left| q(\Dtil) - \sum_{x \in \Dhat} \omega_x q(x) \right| &= \left| \frac1n\sum_{x \in \Dtil} q(x) - \frac1{\hat m} \sum_{x \in \Dhat} g(x) q(x) \right|\\
    &\le \left| \frac1n\sum_{x \in \Dtil} q(x) - q(\mu) \right| + \left|\frac1{\hat m} \left( m q(\mu) - \sum_{x \in \Dhat} g(x) q(x) \right) \right|+\frac{|\hat m - m|}{\hat m} |q(\mu)|\\
    &\le \frac{\alpha}{4} + \frac{\frac{\alpha}{4} m +\frac{\alpha}{4} m}{m-\frac{\alpha}{4} m} \le \alpha.
\end{align*}
Taking a union bound over all $q \in Q$ shows that the probability that any of these events happens is at most \[2|Q|\cdot e^{\!\!-\alpha^2n/8} \!\!+\! (2|Q|\!+\!2) \cdot \exp\!\!\left(\!\frac{-\alpha^2 m}{32 \!\cdot\! e^{\mathrm{D}_2(\mu\|\nu)} \!\!+\!\! \frac{8}{3} \!\cdot\! \alpha \!\cdot\! e^{\mathrm{D}_\infty(\mu\|\nu)}}\!\!\right)\!\!,\]
which is at most $\beta$ if $n$ and $m$ are as large as the theorem requires.
\end{proof}

Having established sufficient conditions for good public data support, we bound the worst-case error of \ours running on a support $\Dsupp$. Since our method is equivalent to running \mwem on a restricted domain $\Dsupp$, its error bound will be similar to that of \mwem. \citet{HardtLM12} show that, if the number of iterations of the algorithm is chosen appropriately, then \mwem has error scaling with $\sqrt{\log(|\cX|)}$ where $\cX$ is the algorithm's data domain. Since \ours is initialized with the restricted data domain $\Xhat$ based on a public dataset of size $m$, its error increases with $\sqrt{\log |\Xhat|} \le \sqrt{\log m}$ instead. Moreover, \ours's error bound includes the best-mixture error $f_{\Dtil, \cQ}(\Xhat)$. Taken together, we present the following bound:

%The next proposition relates the accuracy of \ours\ to the size of the public dataset $m$ and the quality of the support $\Xhat$. Essentially, $\errorLBfunc_{\Dtil,Q}$ measures the error of the optimal reweighting of the support of the public dataset to match the privacy dataset.
% and the theorem below shows that \ours\ attains a reweighting that is close to optimal; the analysis mirrors that of \mwem\ \citep{HardtLM12}.
%
%%%%%%%%%%%%%%%%%%%%%%%%%%%%%
%% Proposition
%%%%%%%%%%%%%%%%%%%%%%%%%%%%%
% \ts{I have edited the the statement below.}
\begin{theorem}\label{thm:acc}
% Suppose we have access to a public dataset with support $\Xhat$.
%
For any private dataset $\Dtil \in \cX^n$, set of statistical queries $Q\subset\{q:\cX\to[0,1]\}$, public dataset $\Dhat \in \cX^m$ with support $\Xhat$, and privacy parameter $\tilde{\varepsilon}>0$, \ours with parameter $T = \Theta\left( \frac{n \tilde{\eps}\sqrt{\log m}}{\log |\cQ|} + \log(1/\beta) \right)$ outputs a distribution $A$ on $\Xhat$ such that, 
with probability $\ge 1-\beta$, 
% \ts{Need to specify how to set $T$.}
\begin{align*}
    % &\max_{q\in \cQ} \abs{q(X) - q(\Dtil)} \\ 
    \max_{q\in \cQ} \abs{q(A) - q(\Dtil)} \leq O\pp{\sqrt{\frac{\log(|Q|) \cdot (\sqrt{\log m} +\log(\tfrac{1}{\beta})) 
   }{n\tilde{\varepsilon}}} + \errorLBfunc_{\Dtil, Q}\pp{\Dsupp}}.
       % \sqrt{\frac{\log\pp{|\Xhat| \pp{\frac{\log(|\cX|/\beta)}{m}}^{\tfrac{1}{2}}  +1 }
    % \log(Q)\log(1/\delta)\log(1/\beta)}{n\tilde{\varepsilon}}} \\ 
\end{align*}
%where $\iota =  \log(|\cQ|)\log(1/\delta)\log(1/\beta) $.%, and 
% $$\gamma_k = \sqrt{\frac{\log(2|\cX|/\beta)}{k}}$$ .
\end{theorem}

\subsection{Privacy Analysis}
The privacy analysis follows from four facts: (i) Permute-and-flip satisfies $\varepsilon_0$-differential privacy \cite{McKennaS20}, which implies $\frac12\varepsilon_0^2$-concentrated differential privacy. (ii) The Gaussian noise addition also satisfies $\frac12\varepsilon_0^2$-concentrated differential privacy. (iii) The composition property of concentrated differential privacy allows us to add up these $2T$ terms \cite{BunS16}. (iv) Finally, we can convert the concentrated differential privacy guarantee into approximate differential privacy \cite{CanonneKS20}.

\begin{theorem}
When run with privacy parameter $\tilde{\varepsilon}>0$, \ours\ satisfies $\frac12\tilde{\varepsilon}^2$-concentrated differential privacy and, for all $\delta>0$, it satisfies$\left( \varepsilon(\delta), \delta\right)$-differential privacy, where 
\begin{align*}
    \varepsilon(\delta) &= \inf_{\alpha>1} \frac12\tilde{\varepsilon}^2\alpha + \frac{\log(1/\alpha\delta)}{\alpha-1} + \log(1-1/\alpha) 
    \le \frac12\tilde{\varepsilon}^2 + \sqrt{2\log(1/\delta)}\cdot\tilde{\varepsilon}.
\end{align*}
\end{theorem}

\endgroup

\section{Empirical Evaluation}

In this section, we presents results comparing \ours against baseline algorithms\footnote{
\hdmm has been considered as a relevant baseline algorithm in past query release work, but having consulted \citet{McKennaMHM18}, we realized that running \hdmm in many settings (including ours) is infeasible. We refer readers to Appendix \ref{app:other_baselines}, where we provide a more detailed discussion.
}
in a variety of settings using the American Census Survey and ADULT datasets.

\subsection{Additional Baseline}\label{subsec:baselines}

\textbf{DualQuery.} Similar to \mwem, \dq \cite{gaboardi2014dual} frames query release as a two-player game, but it reverses the roles of the data and query players. \citet{gaboardi2014dual}
prove theoretical accuracy bounds for \dq that are worse than that of \mwem and show that on low-dimensional datasets where running \mwem is feasible, \mwem outperforms \dq. However, \dq employs optimization heuristics and is often more computationally efficient and scales to a wider range of query release problems than \mwem.

\iffalse
\textbf{HDMM.} Unlike \mwem and \dq, which solve the query release problem by generating synthetic data, the High-Dimensional Matrix Mechanism \cite{McKennaMHM18} is designed to directly answer a workload of queries. By representing query workloads compactly, \hdmm selects a new set of ``strategy'' queries that minimize the estimated error with respect to the input workload. The algorithm then answers the ``strategy'' queries using the \textit{Laplace mechanism} and reconstructs the answers to the input workload queries using these noisy measurements. With the US Census Bureau incorporating \hdmm into its releases \cite{Kifer19}, the algorithm offers a particularly suitable baseline for privately answering statistical queries on the ACS dataset.
\fi

\iffalse
\textbf{\fem.} 
The (Non-Convex)-FTPL with Exponential Mechanism \cite{NRVW} algorithm ... \tl{I'd like to include results for FEM as well but we need to plan this out. I tried it before on PA, but it takes a long time to run and usually does poorly (worse than HDMM and MWEM).}
\fi

\subsection{Data}\label{subsec:data}

\textbf{American Community Survey (ACS).} We evaluate all algorithms on the 2018 American Community Survey (ACS), obtained from the IPUMS USA database \cite{ruggles2020ipums}. Collected every year by the US Census Bureau, the ACS provides statistics that capture the social and economic conditions of households across the country. Given that the Census Bureau may incorporate differential privacy into the ACS after 2025, the data provides a natural testbed for private query release algorithms in a real-world setting.

In total, we select $67$ attributes,\footnote{
An inventory of attributes can be found in Appendix \ref{app:data}.
} giving us a data domain with dimension $287$ and size $\approx 4.99 \times 10^{18}$. To run \mwem, we also construct a lower-dimensional version of the data. We refer to this data domain as ACS (reduced), which has dimension $33$ and a size of $98304$.

For our private dataset $\privD$, we use the 2018 ACS for the state of Pennsylvania (PA-18) and Georgia (GA-18). To select our public dataset $\pubD$, we explore the following:

\textit{Selecting across time.} We consider the setting in which there exists a public dataset describing our population at a different point in time. Using the 2020 US Census release as an example, one could consider using the 2010 US Census as a public dataset for some differentially private mechanism. In our experiments, we use the ACS data for Pennsylvania and Georgia from 2010.
    
\textit{Selecting across states.} We consider the setting in which there exists a public dataset collected concurrently from a different population. In the context of releasing state-level statistics, one can imagine for example that some states have differing privacy laws. In this case, we can identify data for a similar state that has been made public. In our experiments, we select a state with similar demographics to the private dataset's state---Ohio (OH-18) for Pennsylvania and North Carolina (NC-18) for Georgia. To explore how \ours performs using public data from potentially more dissimilar distributions, we also run \ours using the five largest states (by population) according to the 2010 US Census, i.e. California (CA-18), Texas (TX-18), New York (NY-18), Florida (FL-18), and Illinois (IL-18). 

\textbf{ADULT.} We evaluate algorithms on the ADULT dataset from the UCI machine learning dataset repository \cite{Dua:2019}. We construct private and public datasets by sampling with replacement rows from ADULT of size $0.9N$ and $0.1N$ respectively (where $N$ is the number of rows in ADULT). Thus, we frame samples from ADULT as individuals from some population in which there exists both a public and private dataset trying to characterize it (with the former being significantly smaller). In total, the dataset has $13$ attributes, and the data domain has dimension $146$ and support size $\approx 7.32 \times 10^{11}$.

\subsection{Empirical Optimizations}\label{subsec:optimizations}

Following a remark made by \citet{HardtLM12} for optimizing the empirical performance of \mwem, we apply the multiplicative weights update rule using sampled queries $q_i$ and measurements $a_i$ from previous iterations $i$. However, rather than use all past measurements, we choose queries with estimated error above some threshold. Specifically at each iteration $t$, we calculate the term $c_i = |q_i(A_t) - a_i|$ for $i \le t$. In random order, we apply multiplicative weights using all queries and measurements, indexed by $i$, where $c_i \ge \frac{c_t}{2}$, i.e. queries whose noisy error estimates are relatively high. In our implementation of \mwem and \ours, we use this optimization. We also substitute in the \textit{permute-and-flip} and \textit{Gaussian mechanisms} when running \mwem.

\subsection{Hyperparameter tuning}

On the ACS dataset, we select hyperparameters for \ours using $5$-run averages on the corresponding validation sets (treated as private) derived from the 2014 ACS release. Specifically, we evaluate Pennsylvania (PA-14) using PA-10 and OH-14, Georgia (GA-14) using GA-10 and NC-14, and both using CA-14, TX-14, NY-14, FL-14, and IL-14. In all other cases, we simply report the best performing five-run average across all hyperparameter choices. A list of hyperparameters is listed in Table \ref{tab:hyperparameters} in the appendix.

\begin{figure*}[!t]
    \centering
    \includegraphics[width=\linewidth]{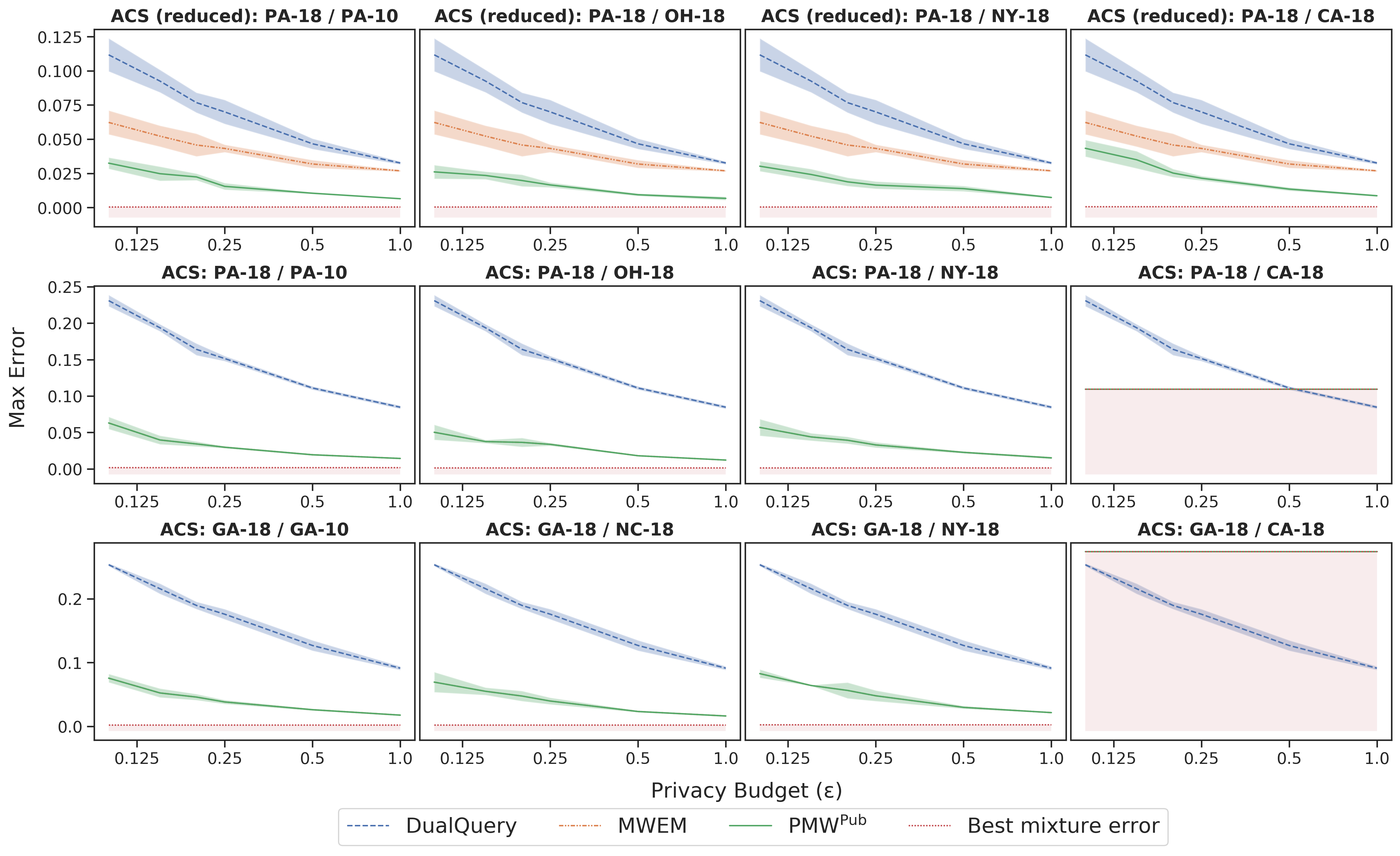}
    \caption{
    Max error for $\varepsilon \in \{ 0.1, 0.15, 0.2, 0.25, 0.5, 1 \}$ and $\delta = \frac{1}{n^2}$. Results are averaged over $5$ runs, and error bars represent one standard error. The \textit{x-axis} uses a logarithmic scale. Given the support of each public dataset, we shade the area below the \textit{best mixture error} to represent max error values that are unachievable by \ours. Additional results using our other choices of public datasets can found in Appendix \ref{app:results}.
    \textbf{Top row:} $5$-way marginals with a workload size of $3003$ (maximum) on the 2018 ACS (reduced) for Pennsylvania.
    \textbf{Middle row:} $3$-way marginals with a workload size of $4096$ on the 2018 ACS for Pennsylvania.
    \textbf{Bottom row:} $3$-way marginals with a workload size of $4096$ on the 2018 ACS for Georgia.
    }
    \label{fig:acs_compare_benchmarks}
\end{figure*}

\subsection{Results}\label{sec:results}

We first present results on the ACS data, demonstrating that \ours achieves state-of-the-art performance in a real world setting in which there exist public datasets that come from slightly different distributions. Next, we run experiments on ADULT and vary how similar the public and private distributions are by artificially changing the proportion of females to males in the public dataset. Finally, we run additional experiments to highlight various aspects of \ours in comparison to the baseline algorithms.

\subsubsection{ACS}\label{sec:results_acs}

In Figure \ref{fig:acs_compare_benchmarks}, we compare \ours against baseline algorithms while using different public datasets. In addition, we plot the best mixture error function for each public dataset to approximate a lower bound on the error of \ours, which we estimate by running (non-private) multiplicative weights with early stopping (at $100$ iterations).

% On the 2018 ACS (reduced) dataset for Pennsylvania, we evaluate on $5$-way marginal queries with the maximum workload size of $3003$. On the full-size 2018 ACS dataset for Pennsylvania and Georgia, we on $3$-way marginal queries with a workload size of $4096$.

We observe that on ACS (reduced) PA-18, \mwem achieves lower error than \dq at each privacy budget (Figure \ref{fig:acs_compare_benchmarks}), supporting the view that \mwem should perform well when it is feasible to run it. Using PA-10, OH-18, and NY-18 as public datasets, \ours improves upon the performance of \mwem and outperforms all baselines. Similarly, on the full-sized ACS datasets for Pennsylvania and Georgia, \ours outperforms \dq.

Next, we present results of \ours when using CA-18 to provide examples where the distribution over the public dataset's support cannot be reweighted to answer all queries accurately. In Figure \ref{fig:acs_compare_benchmarks}, we observe that when using CA-18, \ours performs well on ACS (reduced) PA-18. However, on the set of queries defined for ACS PA-18 and GA-18, the best mixture error for CA-18 is high. Moreover, we observe that across all privacy budgets $\varepsilon$, \ours achieves the best mixture error. Regardless of the number of rounds we run the algorithm for, the accuracy does not improve, and so the error plots in Figure \ref{fig:acs_compare_benchmarks} are flat and have no variance.

While it may be unsurprising that the support over a dataset describing California, a state with relatively unique demographics, is poor for answering large sets of queries on Pennsylvania and Georgia, one would still hope to identify this case ahead of time. One principled approach to verifying the quality of a public dataset is to \textit{spend some privacy budget on measuring its best mixture error}. Given that finding the best mixture error is a sensitivity-$\frac{1}{n}$ query, we can use the \textit{Laplace mechanism} to measure this value. For example, in the cases of both PA and GA (which have size $n\approx10^5$), we can measure the best mixture error with a tiny fraction of the privacy budget (such as $\varepsilon = 0.01$) by adding Laplace noise with standard deviation $\frac{\sqrt{2}}{n\varepsilon} \approx 1.414 \times 10^{-3}$. 

\subsubsection{ADULT}

To provide results on a different dataset, we also run experiments on ADULT in which we construct public and private datasets from the overall dataset. When sampled without bias, the public and private datasets come from the same distribution, and so the public dataset itself already approximates the distribution of the private dataset well. Consequently, we conduct additional experiments by sampling from ADULT according to the attribute \textit{sex} with some bias. Specifically, we sample females with probability $r + \Delta$ where $r\approx0.33$ is the proportion of females in the ADULT dataset. In Figure \ref{fig:adult_compare_to_benchmarks}, we observe that running \ours with a public dataset sampled without bias ($\Delta = 0$) achieves very low error across all privacy budgets, and when using a public dataset sampled with low bias ($|\Delta| \le 0.2$), \ours still outperforms \dq. However, when the public dataset is extremely biased ($\Delta \in \{ 0.45, 0.65 \}$), the performance of \ours deteriorates (though it still significantly outperforms \dq). Therefore, we again show under settings in which the public and private distributions are relatively similar, \ours achieves strong performance.

% We evaluate on $3$-way marginal queries with the maximum workload size of $286$.

\begin{figure}[!t]
    \centering
    \includegraphics[width=\linewidth]{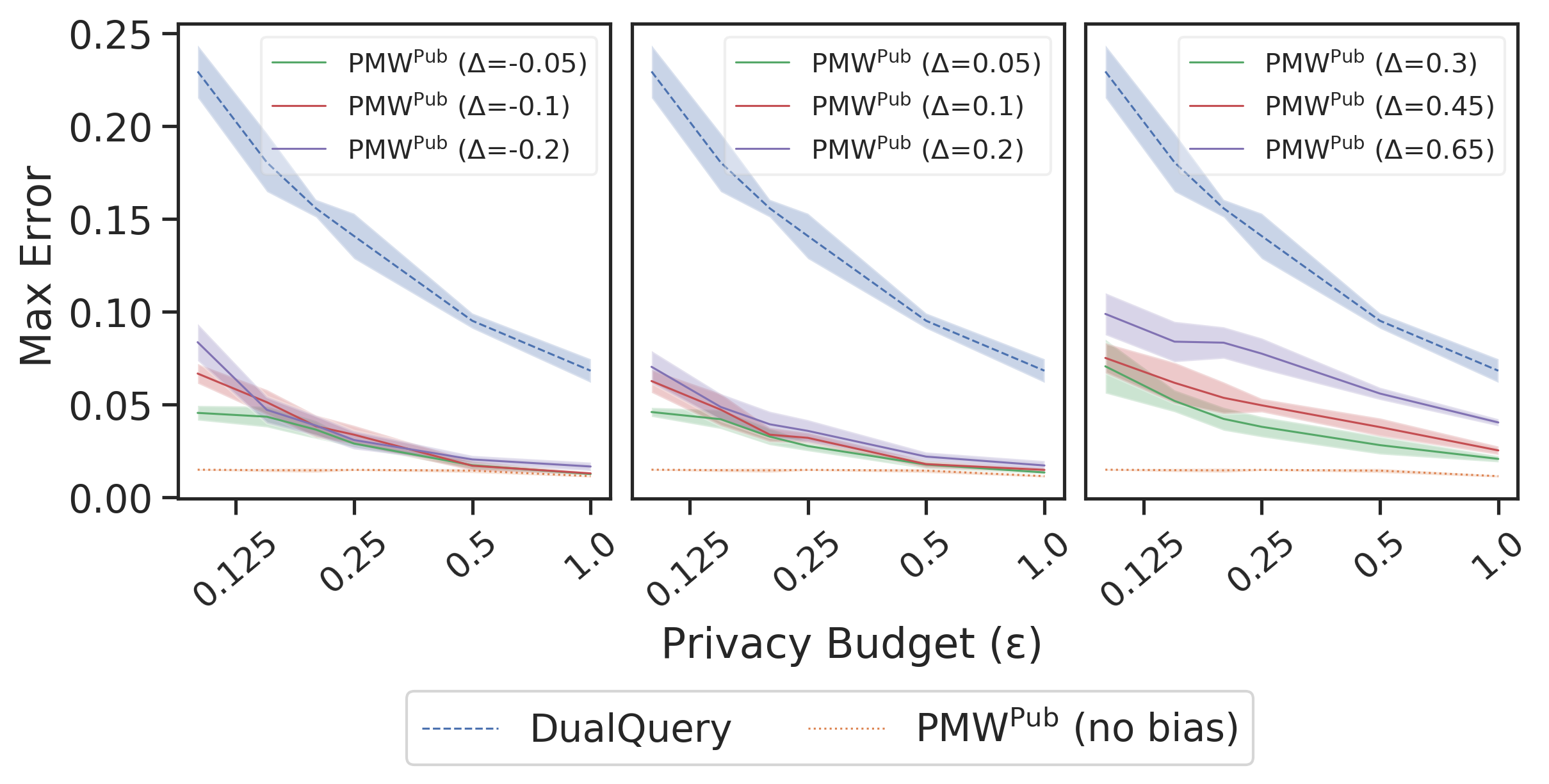}
    \caption{
    Max error on $3$-way marginals across privacy budgets $\varepsilon \in \{ 0.1, 0.15, 0.2, 0.25, 0.5, 1 \}$ where $\delta = \frac{1}{n^2}$ and the workload size is $256$. Results are averaged over $5$ runs, and error bars represent one standard error. Each public dataset is constructed by sampling from ADULT with some bias $\Delta$ over the attribute \textit{sex} (labeled as \ours ($\Delta$)). }
    \label{fig:adult_compare_to_benchmarks}
\end{figure}

% , i.e. rows with the attribute \textit{sex}='Female' are sampled with probability $r - \Delta$ where $r\approx0.33$ is the true proportion of females in the ADULT dataset.

\subsubsection{Additional empirical analysis}

\iffalse
\textit{Workload scalability.} On the 2018 ACS dataset for Pennsylvania, \hdmm scales poorly with respect to workload size when compared to \ours. Figure \ref{fig:acs_compare_workloads} shows that although the maximum error of \hdmm grows significantly as we increase the number of 3-way marginal queries, the maximum error of \ours remains relatively stable. Our experiments suggest that in settings in which the goal is to release very large workloads of queries, \ours may be a more suitable algorithm for achieving high accuracy.

\input{figures/acs/compare_workloads}
\fi

\textit{Public data size requirements.} In Figure \ref{fig:acs_compare_pub_sizes}, we plot the performance on ACS PA-18 of \ours against baseline solutions while varying the fraction of the public dataset used. Specifically, we sample some percentage ($p \in \{100\%, 10\%, 1\%, 0.1\% \}$) of rows from PA-10 and OH-18 to use as the public dataset. \ours outperforms across all privacy budgets, even when only using $1\%$ of the public dataset (Figure \ref{fig:acs_compare_pub_sizes}). From a practical standpoint, these results suggest that one can collect a public dataset that is relatively small (compared to the private dataset) and still achieve good performance using \ours.

\begin{figure}[!t]
    \centering
    \includegraphics[scale=0.6]{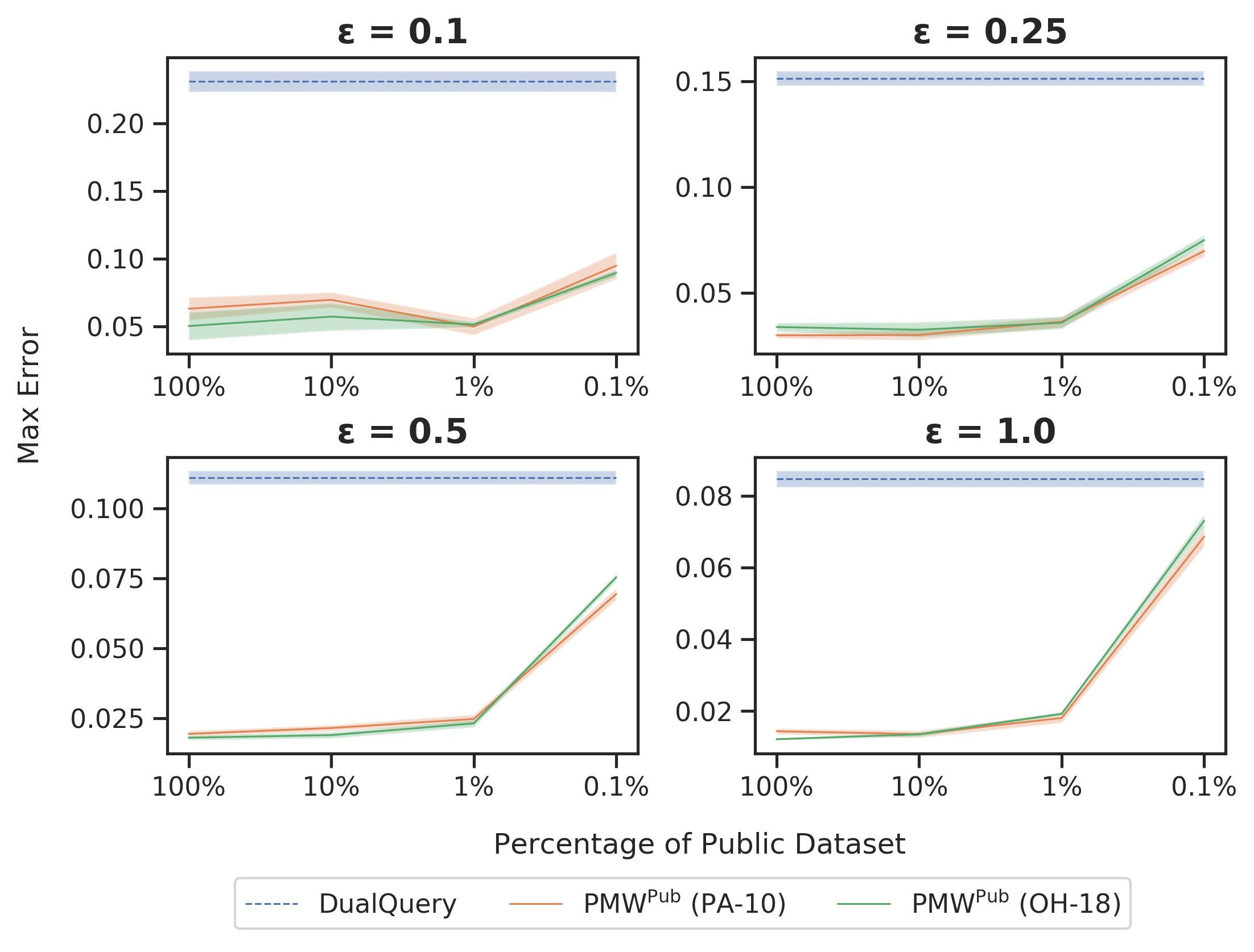}
    \caption{Performance comparison on ACS PA-18 while varying the size of the public dataset. We evaluate on $3$-way marginals with a workload size of $4096$ and privacy budgets defined by $\varepsilon \in \{ 0.1, 0.0.25, 0.5, 1 \}$ and $\delta = \frac{1}{n^2}$.
    }
    \label{fig:acs_compare_pub_sizes}
\end{figure}

\textit{Run-time.} Although running \mwem on the ACS (reduced)-PA dataset is feasible, \ours is computationally more efficient. An empirical evaluation can be found in Table \ref{tab:acs_small_runtime}.

\begin{table}[!t]
\centering
\caption{Run-time comparison between \ours and \mwem on the 2018 ACS PA and ACS (reduced) PA, denoted as \textsc{Full} and \text{Red.} respectively. We compare the per-iteration run-time (in seconds) between \ours (using PA-10 as the public dataset) and \mwem. Experiments are conducted using a single core on an i5-4690K CPU (3.50GHz) machine.}
\setlength\tabcolsep{9pt}
\begin{tabular}{l l c c}
    \toprule
    & \textsc{Algo}. & \textsc{Per-iter. run-time} \\
    \midrule
    \multirow{2}{*}{\textsc{Red.}}
    & \ours & $0.185$ \\
    & \mwem & $0.919$ \\
    % & \hdmm & $-$ & $23.841$ \\
    \midrule
    \multirow{2}{*}{\textsc{Full}}
    & \ours & $2.021$ \\
    & \mwem & $-$ \\
    % & \hdmm & $-$ & $24.236$ \\
    \bottomrule
\end{tabular}
\label{tab:acs_small_runtime}
\end{table}

% However, as a non-iterative algorithm, \hdmm runs significantly faster, presenting a trade-off between the run-time and performance of the two algorithms. 

\section{Conclusion and Discussion}
In this paper, we study differentially private query release in which the privacy algorithm has access to both public and private data samples.  We present an algorithm \ours, a variant of \mwem, that can take advantage of a source of public data.  We demonstrate that \ours improves accuracy over \mwem and other baselines both theoretically and in an empirical study involving the American Community Survey and ADULT datasets.  We also demonstrate that our algorithm is scalable to high-dimensional data. 

Our work also suggests several interesting future directions. Our algorithm \ours essentially leverages public datasets to identify a small support (a subset of the data domain) as a compressed data domain that can be used to run \mwem efficiently. Can we develop algorithms that identify small supports but do not require access to public data? Recent work  \cite{MWD} has taken a similar approach to use differentially private generative adversarial nets (GAN) to generate a support for \mwem. However, their method was not explicitly designed for query release. Another interesting direction is to design hybrid algorithms that alternate between "support generation" steps like private GAN and "sample re-weighting" steps like \mwem. This type of method can potentially construct a better support over time that achieves lower best-mixture-error than the support given by the public datasets.

\section*{Acknowledgments} 
JU is supported by NSF grants CCF-1750640, CNS-1816028, CNS-1916020. ZSW is supported by NSF grant SCC-1952085 and a Google Faculty Research Award.

%\newpage
%\bibliography{main}
%\bibliographystyle{icml2021}
\printbibliography

%\newpage
% \newpage
\onecolumn
\appendix
\section{Appendix}

\subsection{Proofs}\label{app:proofs}

\begingroup
\newcommand{\pp}[1]{\left( #1 \right)}
\newcommand{\bb}[1]{\left[ #1 \right]}
\newcommand{\cc}[1]{\left\{ #1 \right\}}

\newcommand{\Dhat}{\widehat{D}}
\newcommand{\Dtil}{\widetilde{D}}
\newcommand{\error}{\text{error}}
\newcommand{\Xhat}{\widehat{\cX}}

\newcommand{\Dsupp}{\Xhat}
\newcommand{\Ddist}{\mu_{\Dhat}}

\newcommand{\besterror}{\alpha^{-}_{\Dhat}}
\newcommand{\worsterror}{\alpha^{+}_{\Dhat}}

\newcommand{\errorLBfunc}{f}
\newcommand{\DeltaHat}{\widehat{\Delta}}
\newcommand{\DeltaTil}{\widetilde{\Delta}}

\newcommand{\muhat}{\widehat{\mu}}
\newcommand{\mutil}{\widetilde{\mu}}
\newcommand{\KL}{d_{\text{KL}}}

\DeclarePairedDelimiter\abs{\lvert}{\rvert}%
\DeclarePairedDelimiter\norm{\lVert}{\rVert}%
% \newcommand{\l}{\left(}
% \newcommand{\r}{\right)}
% Swap the definition of \abs* and \norm*, so that \abs
% and \norm resizes the size of the brackets, and the 
% starred version does not.
\makeatletter
\let\oldabs\abs
\def\abs{\@ifstar{\oldabs}{\oldabs*}}
\let\oldnorm\norm
\def\norm{\@ifstar{\oldnorm}{\oldnorm*}}
\makeatother

\newcommand{\ex}[2]{{\ifx&#1& \mathbb{E} \else \underset{#1}{\mathbb{E}} \fi \left[#2\right]}}

\begin{proposition}
Let $\tilde{D} \in \cX^n$ and $\Dhat \in \cX^m$. Let $\cQ$ be a finite set of statistical queries $q : \cX \to [0,1]$. Let $\tilde\varepsilon>0$ and $T \in \mathbb{N}$. Let $A$ be the output of Algorithm \ref{alg:framework} with parameters $\tilde\varepsilon$ and $T$, query class $\cQ$, and inputs $\tilde{D}$ as the private dataset and $\Dhat$ as the public dataset.
Then $A$ is a distribution on $\Xhat = \supp(\Dhat) \subset \cX$. For all $\beta \in (0,1)$, if $T \ge 7 \log(3/\beta)$, then \[\pr{\begin{array}{r}\max_{q \in \cQ} |q(\Dtil)-q(A)| \le 2\errorLBfunc_{\Dtil,\cQ}(\Dsupp) + \sqrt{\frac{4\log m}{T} + \frac{4T}{\tilde\varepsilon^2 n^2} +  \frac{4\sqrt{\log(3/\beta)}}{\tilde\varepsilon n }} \\+ \frac{2\sqrt{2T}}{\tilde\varepsilon n} \log |\cQ| + \sqrt{\frac{1}{2T} \log \left( \frac{3}{\beta} \right)}\end{array}} \ge 1-\beta.\]
\end{proposition}
If we set $T=\Theta\left(\frac{\tilde\varepsilon n \sqrt{\log m} }{\log|\cQ|}  + \log(1/\beta)\right)$, then the bound above becomes \[\pr{\max_{q \in \cQ} |q(\Dtil)-q(A)| \le O \left( {\errorLBfunc_{\Dtil,\cQ}(\Dsupp) } + \sqrt{ \frac{\log|\cQ|}{\tilde\varepsilon n} \cdot \left(\sqrt{\log m} + \log(1/\beta)\right) }\right)} \ge 1-\beta,\]
thus proving Theorem \ref{thm:acc}.
\begin{proof}
We follow the analysis of \citet{HardtLM12}. Let $A_t, q_t, a_t$ be as in Algorithm \ref{alg:framework}.
Let $\alpha_0 = \errorLBfunc_{\Dtil,\cQ}(\Dsupp)$ be the error of the optimal reweighting of the public data. 
Let \[D^* = \argmin_{D \in \Delta(\Xhat)} \max_{q \in \cQ} |q(D) - q(\Dtil)|\] be the optimal reweighting so that $\max_{q \in \cQ} |q(D^*)-q(\Dtil)| = \alpha_0$.
We define a potential function $\Psi : \Delta(\Xhat) \to \mathbb{R}$ by \[\Psi(A) = \mathrm{D}_1(D^*\|A) = \sum_{x \in \Xhat} D^*(x) \log\left(\frac{D^*(x)}{A(x)}\right).\]
Since $\Psi$ is a KL divergence, it follows that, for all $A \in \Delta(\Xhat)$, \[ 0 \le \Psi(A) \le \log\left(\frac{1}{\min_{x \in \Xhat} A(x) }\right).\]
In particular, $\Psi(A_T) \ge 0$ and $\Psi(A_0) \le \log m$, since any $x \in \Xhat$ must be one of the $m$ elements of $\Dhat$ and hence has $A_0(x) \ge 1/m$.

Fix an arbitrary $t \in [T]$. For all $x \in \Xhat$, we have $A_t(x) = \frac{A_{t-1}(x)\exp(q_t(x)(a_t-q_t(A_{t-1}))/2)}{\sum_{y \in \Xhat} A_{t-1}(y)\exp(q_t(y)(a_t-q_t(A_{t-1}))/2)}$. Thus
\begin{align*}
    &\Psi(A_{t-1}) - \Psi(A_t) \\
    &= \sum_{x \in \Xhat} D^*(x) \log\left( \frac{A_t(x)}{A_{t-1}(x)} \right)\\
    &= \sum_{x \in \Xhat} D^*(x) \log\left( \frac{\exp(q_t(x)(a_t-q_t(A_{t-1}))/2)}{ \sum_{y \in \Xhat} A_{t-1}(y)\exp(q_t(y)(a_t-q_t(A_{t-1}))/2)} \right)\\
    &= \sum_{x \in \Xhat} D^*(x) q_t(x) \frac{a_t-q_t(A_{t-1})}{2} -  \log\left( \sum_{y \in \Xhat} A_{t-1}(y)\exp\left(q_t(y)\frac{a_t-q_t(A_{t-1})}{2}\right) \right)\\
    &\ge q_t(D^*) \frac{a_t-q_t(A_{t-1})}{2} + 1 -  \sum_{y \in \Xhat} A_{t-1}(y)\exp\left(q_t(y)\frac{a_t-q_t(A_{t-1})}{2}\right) \tag{$\forall x>0 ~~\log x \le x-1$}\\
    &\ge q_t(D^*) \frac{a_t-q_t(A_{t-1})}{2} + 1 -  \sum_{y \in \Xhat} A_{t-1}(y)\left(1+q_t(y)\frac{a_t-q_t(A_{t-1})}{2}+q_t(y)^2\frac{(a_t-q_t(A_{t-1}))^2}{4}\right) \tag{$\forall x \le 1 ~~\exp(x) \le 1+x+x^2$}\\
    &= q_t(D^*) \frac{a_t-q_t(A_{t-1})}{2} + 1 -  1-q_t(A_{t-1})\frac{a_t-q_t(A_{t-1})}{2}-\ex{X \gets A_{t-1}}{q_t(X)^2}\frac{(a_t-q_t(A_{t-1}))^2}{4}\\
    &= (q_t(D^*)-q_t(A_{t-1})) \frac{a_t-q_t(A_{t-1})}{2} -\ex{X \gets A_{t-1}}{q_t(X)^2}\frac{(a_t-q_t(A_{t-1}))^2}{4}\\
    &\ge (q_t(D^*)-q_t(A_{t-1})) \frac{a_t-q_t(A_{t-1})}{2} -\frac{(a_t-q_t(A_{t-1}))^2}{4}\\
    &= \frac14 (2q_t(D^*)-a_t-q_t(A_{t-1}))(a_t-q_t(A_{t-1}))\\
    &= \frac14 (q_t(\Dtil)-q_t(A_{t-1}))^2 + \frac12 (q_t(D^*)-q_t(\Dtil))(a_t-q_t(A_{t-1})) - \frac14 (a_t-q_t(\Dtil))^2\\
    &= \frac14 (q_t(\Dtil)-q_t(A_{t-1}))^2 + \frac12 (q_t(D^*)-q_t(\Dtil))(q_t(\Dtil)-q_t(A_{t-1})) \\&~~~~~~~~~~+ \frac12 (q_t(D^*)-q_t(\Dtil))(a_t-q_t(\Dtil)) - \frac14 (a_t-q_t(\Dtil))^2\\
    &\ge \frac14 (q_t(\Dtil)-q_t(A_{t-1}))^2 - \frac12\alpha_0|q_t(\Dtil)-q_t(A_{t-1})| \\&~~~~~~~~~~+ \frac12(q_t(D^*)-q_t(\Dtil))(a_t-q_t(\Dtil)) - \frac14 (a_t-q_t(\Dtil))^2, \\
\end{align*}
where the final inequality follows from the fact that $|q_t(D^*)-q_t(\Dtil)| \le \alpha_0$ by the definition of $D^*$.

Putting together what we have so far gives
\begin{align*}
    \frac{2}{T} \log m & \ge \frac2T \left(\Psi(A_0) - \Psi(A_T) \right) \\
    &= \frac2T \sum_{t \in [T]} \Psi(A_{t-1}) - \Psi(A_t) \\
    &\ge  \frac2T \sum_{t \in [T]} \frac14(q_t(\Dtil)-q_t(A_{t-1}))^2-\frac2T \sum_{t \in [T]} \frac12\alpha_0 |q_t(\Dtil)-q_t(A_{t-1})| \\&~~~~~~~~~~~~~~~+ \frac2T \sum_{t \in [T]} \frac12(q_t(D^*)-q_t(\Dtil))(a_t-q_t(\Dtil)) -  \frac2T \sum_{t \in [T]} \frac14 (a_t-q_t(\Dtil))^2\\
    &\ge  \frac{1}{2} \left(\frac1T \sum_{t \in [T]} |q_t(\Dtil)-q_t(A_{t-1})|\right)^2-\frac{\alpha_0}{T} \sum_{t \in [T]} |q_t(\Dtil)-q_t(A_{t-1})| \\&~~~~~~~~~~~~~~~+ \frac{1}{T}\sum_{t \in [T]} (q_t(D^*)-q_t(\Dtil))(a_t-q_t(\Dtil)) -  \frac1{2T} \sum_{t \in [T]} (a_t-q_t(\Dtil))^2,
\end{align*}
where the final inequality uses the relationship between the 1-norm and 2-norm.

Now, for each $t \in [T]$ independently, $a_t - q_t(\Dtil)$ is distributed according to $\mathcal{N}(0,1/\varepsilon_0^2 n^2)$. 
Thus the sum $\sum_{t \in [T]} (a_t-q_t(\Dtil))^2$ follows a chi-square distribution with $T$ degrees of freedom and mean $\frac{T}{\varepsilon_0^2 n^2}$. This yields the tail bound \[\forall \kappa \ge 1 ~~~~~\pr{ \sum_{t \in [T]} (a_t-q_t(\Dtil))^2 \ge \kappa \cdot \frac{T}{\varepsilon_0^2 n^2}} \le \left( \kappa \cdot e^{1-\kappa} \right)^{T/2}.\]
In addition, the noise $a_t-q_t(\Dtil)$ is independent from $q_t(D^*)-q_t(\Dtil)$. Hence, the sum $\sum_{t \in [T]} (q_t(D^*)-q_t(\Dtil))(a_t-q_t(\Dtil))$ follows a $\sigma^2$-subgaussian distribution with $\sigma^2 = \frac{1}{\varepsilon_0^2 n^2} \sum_{t \in [T]} (q_t(D^*)-q_t(\Dtil))^2 \le \frac{T \alpha_0^2}{\varepsilon_0^2 n^2}$. In particular, \[\forall \lambda \ge 0 ~~~~~ \pr{\sum_{t \in [T]} (q_t(D^*)-q_t(\Dtil))(a_t-q_t(\Dtil)) \ge \lambda\frac{\alpha_0 \sqrt{T}}{\varepsilon_0 n}} \le e^{-\lambda^2/2}.\]

Set $V:=\frac1T \sum_{t \in [T]} |q_t(\Dtil)-q_t(A_{t-1})|$.

Thus, for all $\kappa,\lambda\ge0$, 
\begin{align*}
    &\pr{\frac12 V^2 - \alpha_0 V \le \frac{2\log m}{T} + \frac{\kappa }{2\varepsilon_0^2 n^2} +  \frac{\lambda \alpha_0}{\varepsilon_0 n \sqrt{T}}} \ge 1-\left( \kappa \cdot e^{1-\kappa} \right)^{T/2}-e^{-\lambda^2/2}.    
\end{align*}

The above expression contains the quadratic inequality 
$\frac12 V^2 - \alpha_0 V \le \frac{2\log m}{T} + \frac{\kappa }{2\varepsilon_0^2 n^2} +  \frac{\lambda \alpha_0}{\varepsilon_0 n \sqrt{T}}.$ This equation implies $$V \le \alpha_0 + \sqrt{\alpha_0^2 + \frac{4\log m}{T} + \frac{\kappa }{\varepsilon_0^2 n^2} +  \frac{2\lambda \alpha_0}{\varepsilon_0 n \sqrt{T}}} \le 2\alpha_0 + \sqrt{\frac{4\log m}{T} + \frac{\kappa }{\varepsilon_0^2 n^2} +  \frac{2\lambda \alpha_0}{\varepsilon_0 n \sqrt{T}}}.$$

Now we invoke the properties of the permute-and-flip or exponential mechanism that selects $q_t$. For each $t \in [T]$, we have \cite[Lemma 7.1]{BassilyNSSSU16} \[\ex{q_t}{|q_t(\Dtil)-q_t(A_{t-1})|} \ge \max_{q \in \cQ} |q(\Dtil)-q(A_{t-1})| - \frac{2}{\varepsilon_0 n} \log |\cQ|.\]
Since $0 \le |q_t(\Dtil)-q_t(A_{t-1})| \le 1$, we can apply Azuma's inequality to obtain \[\pr{\frac1T \sum_{t \in [T]} |q_t(\Dtil)-q_t(A_{t-1})| \ge \frac1T \sum_{t \in [T]} \max_{q \in \cQ} |q(\Dtil)-q(A_{t-1})| - \frac{2}{\varepsilon_0 n} \log |\cQ| - \nu } \ge 1 - e^{-2\nu^2T}\] for all $\nu \ge 0$.
Finally, for $A = \frac1T \sum_{t \in [T]} A_{t-1}$, we have 
\begin{align*}
    \max_{q \in \cQ} |q(\Dtil)-q(A)| &\le \frac1T \sum_{t \in [T]} \max_{q \in \cQ} |q(\Dtil) - q(A_{t-1})|\\
    &\le \frac1T \sum_{t \in [T]} |q_t(\Dtil)-q_t(A_{t-1})| + \frac{2}{\varepsilon_0 n} \log |\cQ| + \nu \tag{with probability $\ge 1- e^{-2\nu^2T})$}\\ 
    &\le 2\alpha_0 + \sqrt{\frac{4\log m}{T} + \frac{\kappa }{\varepsilon_0^2 n^2} +  \frac{2\lambda \alpha_0}{\varepsilon_0 n \sqrt{T}}} + \frac{2}{\varepsilon_0 n} \log |\cQ| + \nu \tag{with probability $\ge 1-\left( \kappa \cdot e^{1-\kappa} \right)^{T/2}-e^{-\lambda^2/2}$}.
\end{align*}
Now we set $\nu = \sqrt{\frac{1}{2T} \log \left( \frac{3}{\beta} \right)}$, $\kappa=2$, and $\lambda=\sqrt{2\log(3/\beta)}$ and apply a union bound. If $T \ge 7 \log(3/\beta)$, then \[\pr{\begin{array}{r} \max_{q \in \cQ} |q(\Dtil)-q(A)| \le 2\alpha_0 + \sqrt{\frac{4\log m}{T} + \frac{2}{\varepsilon_0^2 n^2} +  \frac{2\sqrt{2\log(3/\beta)} \alpha_0}{\varepsilon_0 n \sqrt{T}}} \\ + \frac{2}{\varepsilon_0 n} \log |\cQ| + \sqrt{\frac{1}{2T} \log \left( \frac{3}{\beta} \right)}\end{array}} \ge 1-\beta.\]
Substituting in $\alpha_0 = \errorLBfunc_{\Dtil,\cQ}(\Dsupp) \le 1$ and $\varepsilon_0 = \frac{\tilde\varepsilon}{\sqrt{2T}}$ yields the result.
\end{proof}

We remark that the proof above uses the bound $\Psi(A_0) = \mathrm{D}_1\left(D^*\middle\|\Dhat\right) \le \log m$.
This is tight in the worst case, but is likely to be loose in practice, as the private and public datasets are likely to be relatively similar.
We could also alter Algorithm \ref{alg:framework} to initialize $A_0$ to be uniform on $\Xhat$, in which case we can replace $\log m$ with $\log|\Xhat|$ in the final bound.

\begin{lemma}
For any  support $S\in 2^\cX$ and set of linear queries $Q$, 
the  best mixture error function $f_{D,Q}$ is $\tfrac{1}{n}$ sensitive. 
That is for any pair of neighboring datasets $D,D'$ of size $n$, $\abs{f_{D,Q}(S)  - f_{D',Q}(S)} \leq \tfrac{1}{n}$.
\end{lemma}

\begin{proof}
First, we show that the maximum of $s$-sensitive functions is an $s$-sensitive function and by symmetry the minimum of $s$-sensitive functions is $s$-sensitive. For any $s\leq 1$, let $G = \{ g:\cX\rightarrow [0,1] \}$ be a class of $s$-sensitive functions and define a function $f:\cX \rightarrow [0,1]$ as 
$f(X) = \max_{g\in G} g(X)$, for $X\in\cX$.

Fix any support $S\in 2^\cX$ and neighboring dataset $D,D'$ with size $n$.
Also fix the set $Q$ and note each query $q\in Q$ is bounded in [0,1] and it's $\tfrac{1}{n}$-sensitive.
Let $g' = \text{arg}\max_{g\in G} g(D')$ and $g = \text{arg}\max_{g\in G} g(D) $, then for  neighboring $D, D'$ we have 
\begin{align*}
    f(D) - f(D')  &\leq f(D) - g(D') && \text{Since } f(D') \geq g(D') \\ 
    &\leq f(D) - g(D) + s && \text{Since } |g(D) - g(D')| \leq s \\ 
    &=  s && \text{Since } f(D) = g(D) 
\end{align*}
Similarly, we can show that $ f(D') - f(D)\leq s$, therefore $f$ is $s$-sensitive.

Since a marginal query $q\in \cQ$, is $\tfrac{1}{n}$-sensitive, after fixing any $\mu$ the expression \[\max_{q\in Q}\left|q(D) -  \sum_{x \in S} \mu_x  q(x) \right|\] is a max of $\tfrac{1}{n}$ sensitive functions, then by the argument above it is a $\tfrac{1}{n}$-sensitive function. It follows that $f_{D,Q}(S)$ is a minimum of $\tfrac{1}{n}$-sensitive functions therefore $f_{D,Q}(S)$ is $\tfrac{1}{n}$-sensitive.

\end{proof}

\subsection{Data}\label{app:data}

Attributes for our experiments on ACS, ACS (reduced), and ADULT:
\begin{itemize}
    \itemsep0em 
    \item \textbf{ACS:} ACREHOUS, AGE, AVAILBLE, CITIZEN, CLASSWKR, DIFFCARE, DIFFEYE, DIFFHEAR, DIFFMOB, DIFFPHYS, DIFFREM, DIFFSENS, DIVINYR, EDUC, EMPSTAT, FERTYR, FOODSTMP, GRADEATT, HCOVANY, HCOVPRIV, HINSCAID, HINSCARE, HINSVA, HISPAN, LABFORCE, LOOKING, MARRINYR, MARRNO, MARST, METRO, MIGRATE1, MIGTYPE1, MORTGAGE, MULTGEN, NCHILD, NCHLT5, NCOUPLES, NFATHERS, NMOTHERS, NSIBS, OWNERSHP, RACAMIND, RACASIAN, RACBLK, RACE, RACOTHER, RACPACIS, RACWHT, RELATE, SCHLTYPE, SCHOOL, SEX, SPEAKENG, VACANCY, VEHICLES, VET01LTR, VET47X50, VET55X64, VET75X90, VET90X01, VETDISAB, VETKOREA, VETSTAT, VETVIETN, VETWWII, WIDINYR, WORKEDYR
    \item \textbf{ACS (reduced):} DIFFEYE, DIFFHEAR, EMPSTAT, FOODSTMP, HCOVPRIV, HINSCAID, HINSCARE, OWNERSHP, RACAMIND, RACASIAN, RACBLK, RACOTHER, RACPACIS, RACWHT, SEX
    \item \textbf{ADULT:} sex, income\textgreater50K, race, relationship, marital-status, workclass, occupation, education-num, native-country, capital-gain, capital-loss, hours-per-week, age
\end{itemize}

\noindent In addition, we discretize the following continuous attributes (with the number of bins after preprocessing) into categorical attributes:
\begin{itemize}
    \itemsep0em 
    \item \textbf{ACS:} AGE ($10$)
    \item \textbf{ACS (reduced):} AGE ($10$)
    \item \textbf{ADULT:} capital-gain ($16$), capital-loss ($6$), hours-per-week ($10$), age ($10$)
\end{itemize}

\subsection{Hyperparameters}\label{appendix:hyperparameters}

We report hyperparameters used across all experiments in Table \ref{tab:hyperparameters}. 
% Note that \hdmm does not have hyperparameters.

\begin{table}[!h]
\centering
\caption{Hyperparameter selection for experiments on all datasets.}
\label{tab:hyperparameters}
\begin{tabular}{l l c}
    \toprule
    Method & Parameter & Values \\
    \midrule
    \multirow{3}{*}{\ours} & \multirow{3}{*}{$T$} 
    & 300, 250, 200, 150, \\
    & & 125, 100, 75, 50, \\
    & & 25, 10, 5 \\
    \midrule
    \multirow{3}{*}{\mwem} & \multirow{3}{*}{$T$} 
    & 300, 250, 200, 150, \\
    & & 125, 100, 75, 50, \\
    & & 25, 10, 5 \\
    \midrule
    \multirow{2}{*}{\dq}
    & samples & 500 250 100 50 \\
    & $\eta$ & 5 4 3 2 \\
    \bottomrule
\end{tabular}
\end{table}

\subsection{Experimental results}\label{app:results}

\begin{figure*}[!t]
    \centering
    \includegraphics[width=\linewidth]{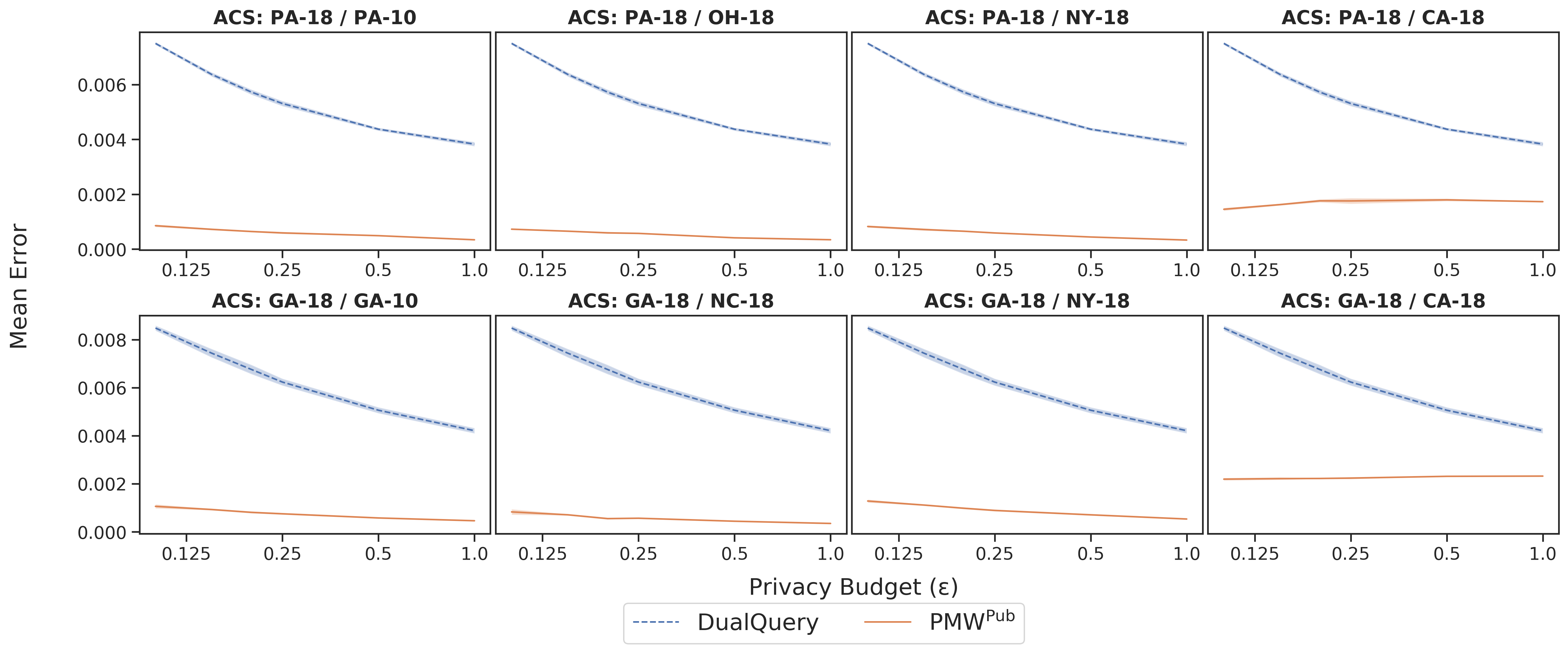}
    \caption{
    Mean error for $\varepsilon \in \{ 0.1, 0.15, 0.2, 0.25, 0.5, 1 \}$ and $\delta = \frac{1}{n^2}$, evaluated on $3$-way marginals with a workload size of $4096$. Results are averaged over $5$ runs, and error bars represent one standard error. The \textit{x-axis} uses a logarithmic scale. 
    \textbf{Top row:} 2018 ACS for Pennsylvania.
    \textbf{Bottom row:} 2018 ACS for Georgia.
    }
    \label{fig:acs_compare_benchmarks_mean}
\end{figure*}

\begin{figure*}[!t]
    \centering
    \includegraphics[width=\linewidth]{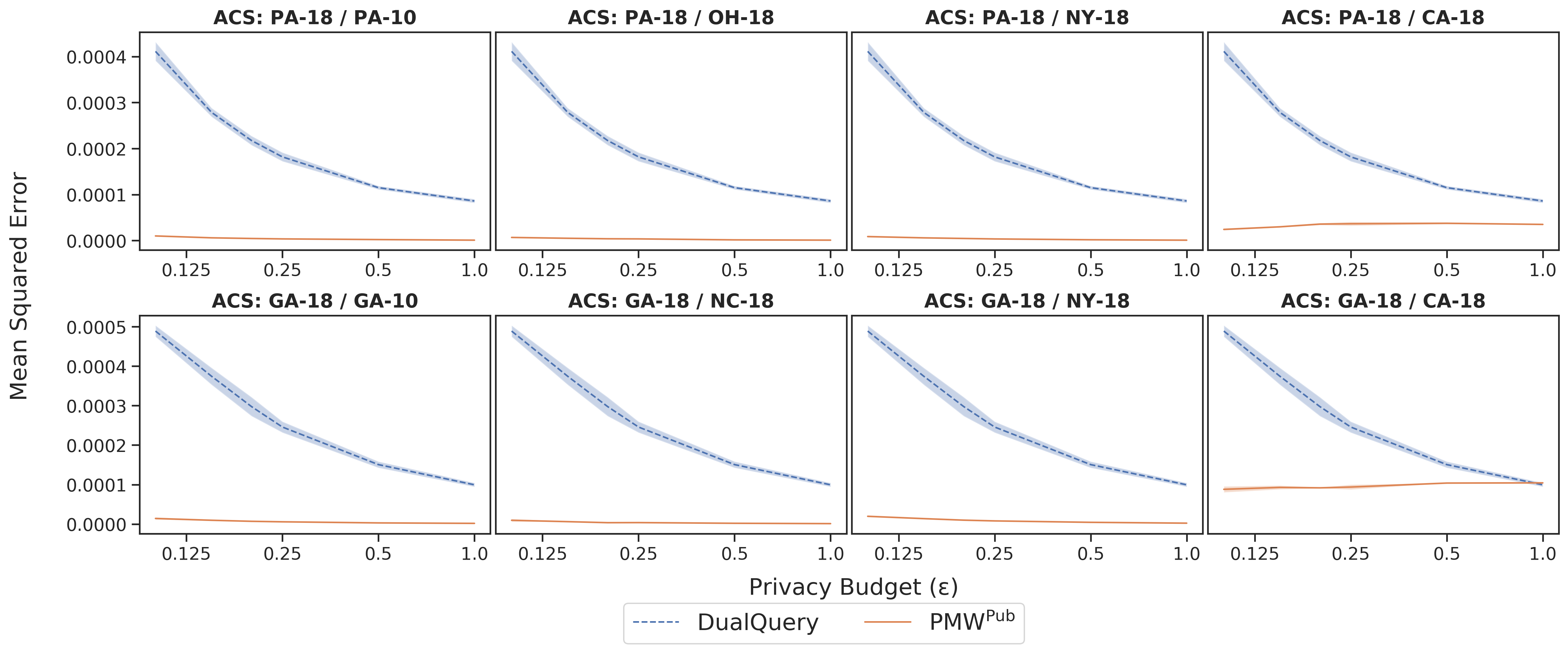}
    \caption{
    Mean squared error for $\varepsilon \in \{ 0.1, 0.15, 0.2, 0.25, 0.5, 1 \}$ and $\delta = \frac{1}{n^2}$, evaluated on $3$-way marginals with a workload size of $4096$. Results are averaged over $5$ runs, and error bars represent one standard error. The \textit{x-axis} uses a logarithmic scale. 
    \textbf{Top row:} 2018 ACS for Pennsylvania.
    \textbf{Bottom row:} 2018 ACS for Georgia.
    }
    \label{fig:acs_compare_benchmarks_mse}
\end{figure*}

\begin{figure*}[!t]
    \centering
    \includegraphics[width=\linewidth]{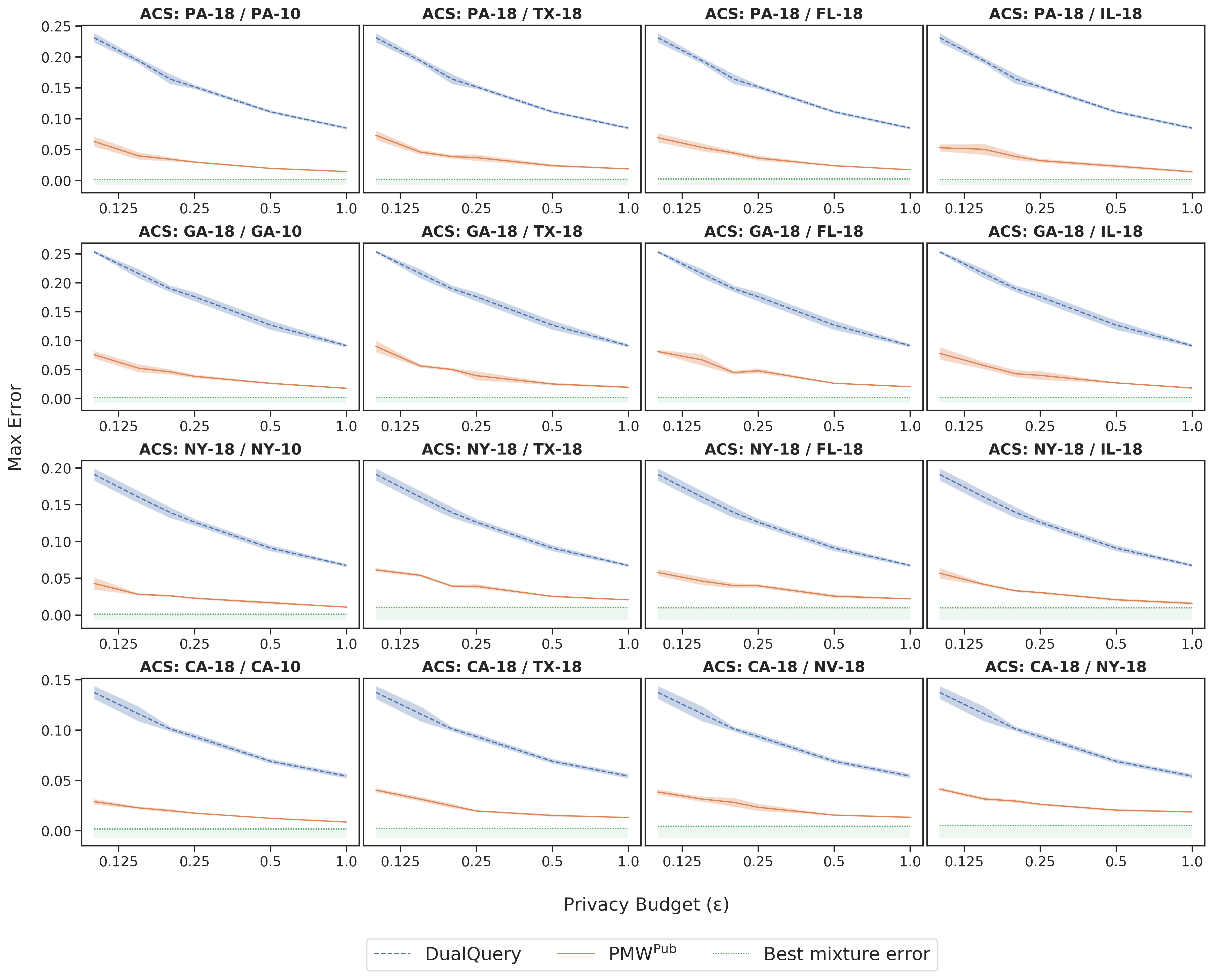}
    \caption{
    Additional plots of the max error (3-way marginals and workload size of 4096) for $\varepsilon \in \{ 0.1, 0.15, 0.2, 0.25, 0.5, 1 \}$ and $\delta = \frac{1}{n^2}$ on PA-18 (Row 1), GA-18 (Row 2), NY-18 (Row 3), and CA-18 (Row 4). Results are averaged over $5$ runs, and error bars represent one standard error. The \textit{x-axis} uses a logarithmic scale. Given the support of each public dataset, we shade the area below the \textit{best mixture error} to represent max error values that are unachievable by \ours.
    }
    \label{fig:acs_compare_benchmarks_additional}
\end{figure*}

\begin{figure}[!t]
    \centering
    \includegraphics[scale=0.6]{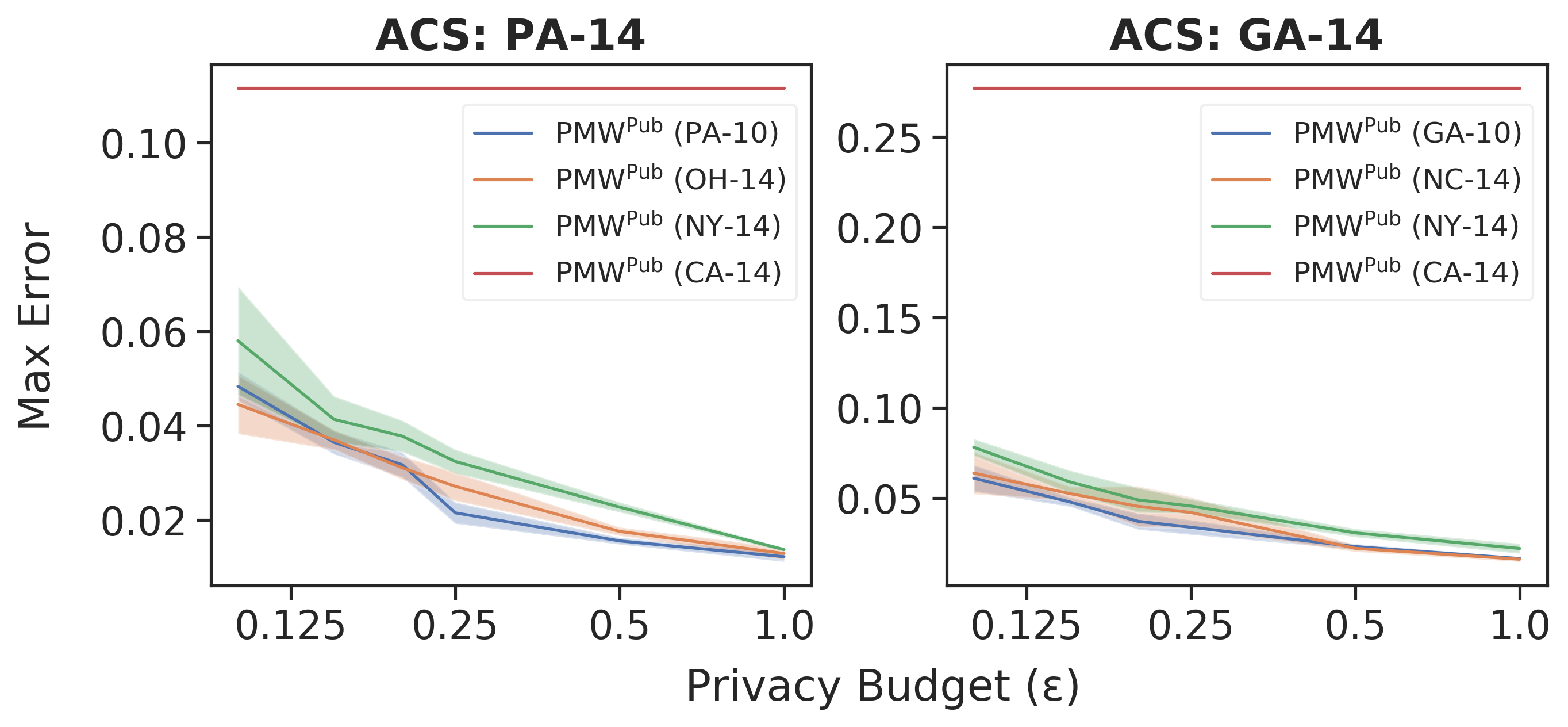}
    \caption{
    Max error on the ACS validation sets for $3$-way marginals with a workload size of $4096$ with privacy $\varepsilon \in \{ 0.1, 0.15, 0.2, 0.25, 0.5, 1 \}$ and $\delta = \frac{1}{n^2}$. Results are averaged over $5$ runs, and error bars represent one standard error. The \textit{x-axis} uses a logarithmic scale.
    \textbf{Left:} 2014 ACS for Pennsylvania.
    \textbf{Right:} 2014 ACS for Georgia.
    }
    \label{fig:acs_compare_pub_val}
\end{figure}

\begin{table*}[!t]
\centering
\caption{Max error (averaged over 5 runs, best results in \textbf{bold}) comparison on the 2018 ACS (reduced)-PA, 2018 ACS-PA, and 2018 ACS-GA datasets. At each privacy budget parametrized by $\varepsilon \in \{ 0.1, 0.15, 0.2, 0.25, 0.5, 1 \}$ and $\delta = \frac{1}{n^2}$, \ours uses the public dataset (and corresponding hyperparameter $T$) that achieves the lowest max error on the validation set.}
\setlength\tabcolsep{7pt}
\begin{small}
\begin{tabular}{l l c c c c c c }
    \toprule
    \textsc{Dataset} & \textsc{Algo}. & $\varepsilon=0.1$ & $\varepsilon=0.15$ & $\varepsilon=0.2$ & $\varepsilon=0.25$ & $\varepsilon=0.5$ & $\varepsilon=1$ \\
    \midrule
    \multirow{2}{*}{\textsc{ACS (red.)-PA}}
    & \ours & $\mathbf{0.0301}$ & $\mathbf{0.0197}$ & $\mathbf{0.0196}$ & $\mathbf{0.0172}$ & $\mathbf{0.0097}$ & $\mathbf{0.0067}$ \\
    % & \hdmm & $0.1883$ & $0.1267$ & $0.0951$ & $0.0782$ & $0.0392$ & $0.0193$ \\
    & \dq & $0.1115$ & $0.0871$ & $0.0816$ & $0.0625$ & $0.0473$ & $0.0330$ \\
    \midrule
    \multirow{2}{*}{\textsc{ACS-PA}}
    & \ours & $\mathbf{0.0499}$ & $\mathbf{0.0458}$ & $\mathbf{0.0332}$ & $\mathbf{0.0298}$ & $\mathbf{0.0195}$ & $\mathbf{0.0141}$ \\
    % & \hdmm & $0.2360$ & $0.1506$ & $0.1104$ & $0.0908$ & $0.0497$ & $0.0233$ \\
    & \dq & $0.2289$ & $0.1908$ & $0.1639$ & $0.1526$ & $0.1086$ & $0.0816$ \\
    \midrule
    \multirow{2}{*}{\textsc{ACS-GA}}
    & \ours & $\mathbf{0.0753}$ & $\mathbf{0.0523}$ & $\mathbf{0.0470}$ & $\mathbf{0.0380}$ & $\mathbf{0.0244}$ & $\mathbf{0.0175}$ \\
    % & \hdmm & $0.2939$ & $0.1972$ & $0.1500$ & $0.1190$ & $0.0581$ & $0.0306$ \\
    & \dq & $0.2615$ & $0.2117$ & $0.1904$ & $0.1709$ & $0.1212$ & $0.0910$ \\
    \bottomrule
\end{tabular}
\end{small}
\label{tab:acs_compare_to_benchmarks_validated}
\end{table*}

\begin{figure*}[!t]
    \centering
    \includegraphics[scale=0.35]{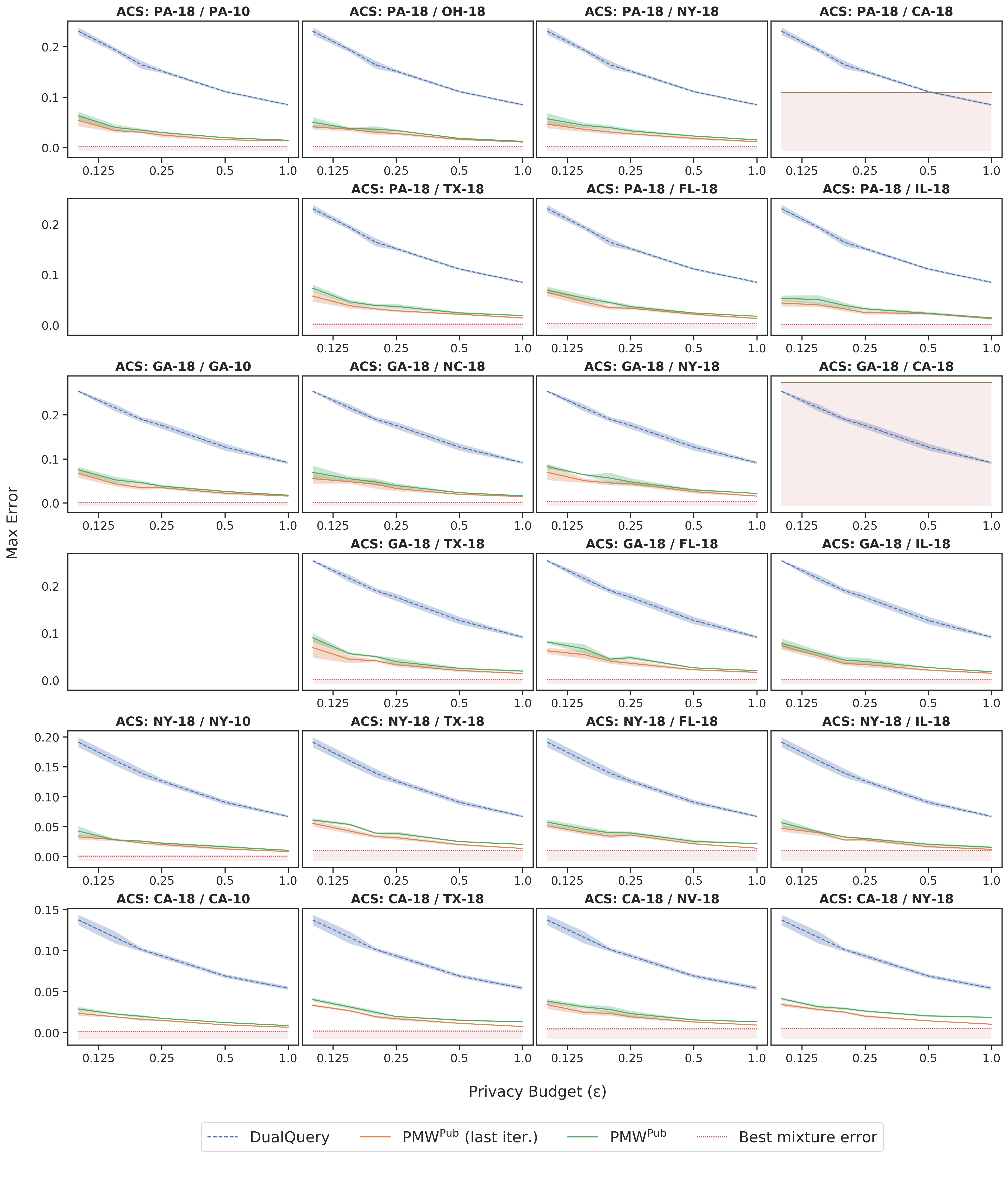}
    \caption{
    We compare \ours with the variant of \ours that outputs the last iterate $A_T$ for all experiments (3-way marginals and workload size of 4096) on the (full-sized) 2018 ACS dataset, plotting max error for $\varepsilon \in \{ 0.1, 0.15, 0.2, 0.25, 0.5, 1 \}$ and $\delta = \frac{1}{n^2}$. Results are averaged over $5$ runs, and error bars represent one standard error. The \textit{x-axis} uses a logarithmic scale. Given the support of each public dataset, we shade the area below the \textit{best mixture error} to represent max error values that are unachievable by \ours. Using the last iterate in \ours improves performance across all experiments.
    }
    \label{fig:acs_compare_benchmarks_last_iter}
\end{figure*}

\subsubsection{Evaluations using mean error and RMSE}

In Figures \ref{fig:acs_compare_benchmarks_mean} and \ref{fig:acs_compare_benchmarks_mse}, we evaluate on ACS PA-18 and ACS GA-18 with respect to mean error and mean squared error respectively. Like in Figure \ref{fig:acs_compare_benchmarks}, in which we present results with respect to max error, \ours performs well when using public datasets with low best mixture error. In the case where \ours uses a poor dataset (i.e., CA-18), we observe the performance of \ours suffers, and the algorithm does not improve even when the privacy budget is increased. However, we note that \ours still performs better than \dq in this setting (where $\varepsilon \le 1.0$).

\subsubsection{Additional results}

In Figure \ref{fig:acs_compare_benchmarks_additional}, we plot results for ACS PA-18 and ACS GA-18 comparing \ours using the 2010 ACS data (PA-10 and GA-10) with the remaining public datasets (TX-18, FL-18, IL-18) not presented in Figure \ref{fig:acs_compare_benchmarks}. In addition, we present results on 2018 ACS data for the states of New York (NY-18) and California (CA-18). To run \ours, we choose Texas (TX-18), Florida (FL-18), and Illinois (IL-18) for New York and choose Texas (TX-18), Nevada (NV-18), and New York (NY-18) for California.

\subsubsection{Using the last iterate}

In this work, we present theoretical guarantees of \ours in which we output the average distribution $A = \text{avg}_{t \le T} A_t$ (see Algorithm \ref{alg:framework}), mimicking the output in the original formulation of \mwem. However, \citet{HardtLM12} note that while they prove guarantees for this variant of \mwem, in practical settings, one can often achieve better results by outputting the distribution from the last iterate, $A_T$. In Figure \ref{fig:acs_compare_benchmarks_last_iter}, we compare \ours to the variant of \ours that outputs $A_T$ and observe that indeed, outputting the last iterate achieves better performance across all experiments (excluding those in which the best mixture error of the public dataset's support is high, i.e. CA-18).

\subsubsection{Identifying public datasets with poor support}

In Section \ref{sec:results_acs}, we describe how using Laplace noise, one can get determine the quality of a support by getting a noisy estimate of the best mixture error for any public dataset. While we emphasize that this strategy is the most principled approach to ensuring the public data is viable for \ours, we note that in settings like ours in which we have a validation set, one can apply additional sanity checks. For instance, in Figure \ref{fig:acs_compare_pub_val}, we observe that \ours performs poorly on the validation set when using CA-14, both in absolute terms and relative to the other public datasets. For demonstration purposes, we show in Table \ref{tab:acs_compare_to_benchmarks_validated} that if we select the public dataset (at each privacy budget $\varepsilon$) based solely on which public dataset performed best on the validation set, we achieve very strong results. Thus in practical settings, one can use validation sets in conjunction with the best mixture error function to find a suitable public dataset (for example, one can first filter out poor public datasets using a validation set and then find the best mixture error of any remaining candidates).

% \subsubsection{Run-time}

% To numerically compare the computational efficiency of \ours vs. \hdmm, we present run-times on the ACS (reduced)-PA dataset in Table \ref{tab:acs_small_runtime}.

\subsection{Discussion of other baselines}\label{app:other_baselines}

\paragraph{Ji \& Elkan (2013).}
While \citet{ji2013differential}'s method reweights the support of a public dataset, their goal is not tailored towards query release. \citet{ji2013differential} instead measure the success their algorithm by evaluating the parameters learned from training a support vector machine on the synthetic dataset. Furthermore, they specifically contrast their method's goal with that of \mwem, whose objective is to optimize over a set of predefined queries. Thus, one would expect the method to perform worse than \ours for query release. To verify this hypothesis, we implement their algorithm with hyperparameter $\lambda \in \{0.005, 0.01, 0.025, 0.05, 0.1, 0.5\}$. \ours outperforms across all metrics (max, mean, and mean squared error) and privacy budgets on ACS PA-18 and GA-18 using each public dataset used to evaluate \ours in Section \ref{sec:results_acs}. For example, w.r.t. max error on PA-18, \ours outperforms by between $1.38 \times$ and $5.07 \times$ (depending on $\varepsilon$) when using PA-10 as the public dataset.

\paragraph{\hdmm.}
Unlike \mwem and \dq, which solve the query release problem by generating synthetic data, the High-Dimensional Matrix Mechanism \cite{McKennaMHM18} is designed to directly answer a workload of queries. By representing query workloads compactly, \hdmm selects a new set of ``strategy'' queries that minimize the estimated error with respect to the input workload. The algorithm then answers the ``strategy'' queries using the \textit{Laplace mechanism} and reconstructs the answers to the input workload queries using these noisy measurements.

We had originally evaluated against \hdmm. However, having consulted \citet{McKennaMHM18}, we learned that currently, running \hdmm is infeasible for the ACS and ADULT datasets. There does not exist a way to solve the least square problem described in the paper for domain sizes larger than approximately $10^9$. While a variant of \hdmm using local least squares could potentially circumvent such computational issues, there does not exist support for this version of \hdmm for more general query workloads outside of those predefined in \citet{McKennaMHM18}'s codebase. Currently, there is no timeline for when the authors will begin developing this modification that would allow us to use \hdmm as a baseline for the ADULT and ACS datasets.

% \footnote{
% We originally used an implementation from \url{https://github.com/giusevtr/fem} but have since learned that this code was incorrect and that it is not computationally feasible to run \hdmm in \citet{vietri2020new}'s experimental setting either.
% }

\end{document}